\documentclass{article}

\usepackage{microtype}
\usepackage{graphicx}
\usepackage{subcaption}
\usepackage{booktabs} 
\usepackage[T1]{fontenc}
\usepackage{hyperref}

\usepackage[accepted]{icml2021}

\usepackage{amsfonts}
\usepackage{array}
\usepackage{multirow}
\usepackage{threeparttable}
\usepackage{amsmath, amsthm, amssymb}
\DeclareMathOperator*{\argmax}{arg\,max}

\newtheorem{prop}{Proposition}

\icmltitlerunning{Sliced Iterative Normalizing Flows}

\begin{document}

\twocolumn[
\icmltitle{Sliced Iterative Normalizing Flows}



\icmlsetsymbol{equal}{*}

\begin{icmlauthorlist}
\icmlauthor{Biwei Dai}{ucb}
\icmlauthor{Uro{\v s} Seljak}{ucb,lbl}
\end{icmlauthorlist}

\icmlaffiliation{ucb}{Department of Physics, University of California, Berkeley, California, USA}
\icmlaffiliation{lbl}{Lawrence Berkeley National Laboratory, Berkeley, California, USA}

\icmlcorrespondingauthor{Biwei Dai}{biwei@berkeley.edu}

\icmlkeywords{Generative Models, Normalizing Flow, Optimal Transport, Sliced Wasserstein Distance}

\vskip 0.3in
]



\printAffiliationsAndNotice{}  

\begin{abstract}
We develop an iterative (greedy) deep learning (DL) algorithm
which is able to transform an arbitrary probability distribution function (PDF) into the target PDF. The model is based on iterative Optimal Transport of a series of 1D slices, matching on each slice the marginal PDF to the target. 
The axes of the orthogonal slices are chosen to maximize the PDF difference using Wasserstein distance at each iteration, which 
enables the algorithm to scale well to high dimensions. 
As special cases of this algorithm, we introduce two sliced iterative Normalizing Flow (SINF) models, which map from the data to the latent space (GIS) and vice versa (SIG). 
We show that SIG is able to generate high quality samples of image datasets, which match the GAN benchmarks, while GIS obtains competitive results on density estimation tasks compared to the density trained NFs, and is more stable, faster, and achieves higher $p(x)$ when trained on small training sets. SINF
approach deviates significantly from the current DL paradigm, as it is greedy and does not use concepts such as mini-batching, stochastic gradient descent and gradient back-propagation through deep layers. 

\end{abstract}

\section{Introduction}

\label{sec:introduction}

Latent variable generative models such as Normalizing Flows (NFs) \citep{rezende2015variational,dinh2014nice,dinh2016density,kingma2018glow}, Variational AutoEncoders (VAEs) \citep{kingma2013auto,rezende2014stochastic} and Generative Adversarial Networks (GANs) \citep{goodfellow2014generative, radford2015unsupervised} aim to model the distribution $p(x)$ of high-dimensional input data $x$ by introducing a mapping from a latent variable $z$ to $x$, where $z$ is assumed to follow a given prior distribution $\pi(z)$. These models usually parameterize the mapping using neural networks, and the training of these models typically consists of minimizing a dissimilarity measure between the model distribution and the target distribution. For NFs and VAEs, maximizing the marginal likelihood is equivalent to minimizing the Kullback–Leibler (KL) divergence. While for GANs, the adversarial training leads to minimizations of the Jenson-Shannon (JS) divergence \citep{goodfellow2014generative}. The performance of these models largely depends on the following aspects: \newline\ \ \ \textbf{1)} The parametrization of the mapping (the architecture of the neural network) should match the structure of the data and be expressive enough. Different architectures have been proposed \citep{kingma2018glow, van2017neural,karras2017progressive, karras2019style}, but to achieve the best performance on a new dataset one still needs extensive hyperparameter explorations \citep{lucic2018gans}.\newline \ \ \  \textbf{2)} The dissimilarity measure (the loss function) should be appropriately chosen for the tasks. For example, in high dimensions the JS divergence is more correlated with the sample quality than KL divergence \citep{huszar2015not, theis2015note}, which is believed to be one of the reasons that GANs are able to generate higher quality samples than VAEs and NFs. However, JS divergence is hard to directly work with, and the adversarial training could bring many problems such as vanishing gradient, mode collapse and non-convergence \citep{arjovsky2017towards, wiatrak2019stabilizing}.

To avoid these complexities, in this work we adopt a different approach to build the map from latent variable $z$ to data $x$. We approach this problem from the Optimal Transport (OT) point of view. OT studies whether the transport maps exist between two probability distributions, and if they do, how to construct the map to minimize the transport cost. Even though the existence of transport maps can be proved under mild conditions \citep{villani2008optimal}, it is in general hard to construct them in high dimensions. We propose to decompose the high dimensional problem into a succession of 1D transport problems, where the OT solution is known. The mapping is iteratively augmented, and it has a NF structure that allows explicit density estimation and efficient sampling. We name the algorithm Sliced Iterative Normalizing Flow (SINF). Our objective function is inspired by the Wasserstein distance, which is defined as the minimal transport cost and has been widely used in the loss functions of generative models \citep{arjovsky2017towards, tolstikhin2017wasserstein}. We propose a new metric, max K-sliced Wasserstein distance, which enables the algorithm to scale well to high dimensions. 

In particular, SINF algorithm has the following properties: \newline
    \textbf{1)} The performance is competitive compared to state-of-the-art (SOTA) deep learning generative models. We show that if the objective is optimized in data space, the model is able to produce high quality samples similar to those of GANs; and if it is optimized in latent space, the model achieves comparable performance on density estimation tasks compared to NFs trained with maximum likelihood, and achieves highest performance on small training sets.\newline
    \textbf{2)} Compared to generative models based on neural networks, this algorithm has very few hyperparameters, and the performance is insensitive to their choices. \newline
    \textbf{3)} The model training is very stable and insensitive to random seeds. In our experiments we do not observe any cases of training failures. \newline

\section{Background}

\subsection{Normalizing Flows}

Flow-based models provide a powerful framework for density estimation \citep{dinh2016density, papamakarios2017masked}
and sampling \citep{kingma2018glow}. These models map the $d$-dimensional data $x$ to $d$-dimensional latent variables $z$ through a sequence of invertible transformations $f = f_1 \circ f_2 \circ ... \circ f_L$, such that $z = f(x)$ and $z$ is mapped to a base distribution $\pi(z)$, which is normally chosen to be a standard Normal distribution. The probability density of data $x$ can be evaluated using the change of variables formula:
\begin{eqnarray}
    \label{eq:flow}
    p(x) =& \pi(f(x)) |\det \left(\frac{\partial f(x)}{\partial x}\right)| \nonumber \\
    =& \pi(f(x)) \prod_{l=1}^L |\det \left(\frac{\partial f_l(x)}{\partial x}\right)| .
\end{eqnarray}
The Jacobian determinant $\det (\frac{\partial f_l(x)}{\partial x})$ must be easy to compute for evaluating the density, and the transformation $f_l$ should be easy to invert for efficient sampling.


\subsection{Radon Transform}
\label{subsec:radon}

Let $\mathbb{L}^1(X)$ be the space of absolute integrable functions on $X$. The Radon transform $\mathcal{R}:\mathbb{L}^1(\mathbb{R}^d) \to \mathbb{L}^1(\mathbb{R} \times \mathbb{S}^{d-1})$ is defined as 
\begin{equation}
    \label{eq:R}
    (\mathcal{R}p)(t,\theta) = \int_{\mathbb{R}^d} p(x) \delta(t-\langle x,\theta \rangle)dx ,
\end{equation}
where $\mathbb{S}^{d-1}$ denotes the unit sphere $\theta_1^2+ \cdots \theta_d^2 = 1$ in $\mathbb{R}^d$, $\delta(\cdot)$ is the Dirac delta function, and $\langle \cdot,\cdot \rangle$ is the standard inner product in $\mathbb{R}^d$. For a given $\theta$, the function $(\mathcal{R}p)(\cdot,\theta):\mathbb{R} \to \mathbb{R}$ is essentially the slice (or projection) of $p(x)$ on axis $\theta$.

Note that the Radon transform $\mathcal{R}$ is invertible. Its inverse, also known as the filtered back-projection formula, is given by \citep{helgason2010integral, kolouri2019generalized}
\begin{equation}
    \label{eq:inverseRadon}
    \mathcal{R}^{-1}((\mathcal{R}p)(t,\theta)) (x) = \int_{\mathbb{S}^{d-1}} ((\mathcal{R}p)(\cdot,\theta) * h)(\langle x,\theta \rangle) d\theta,
\end{equation}
where $*$ is the convolution operator, and the convolution kernel $h$ has the Fourier transform $\hat{h}(k)=c|k|^{d-1}$. The inverse Radon transform provides a practical way to reconstruct the original function $p(x)$ using its 1D slices $(\mathcal{R}p)(\cdot,\theta)$, and is widely used in medical imaging. This inverse formula implies that if the 1D slices of two functions are the same in all axes, these two functions are identical, also known as Cram{\'e}r-Wold theorem \citep{cramer1936some}.

\subsection{Sliced and Maximum Sliced Wasserstein Distances}
\label{subsec:SWD}

The p-Wasserstein distance, $p\in [1,\infty)$, between two probability distributions $p_1$ and $p_2$ is defined as:
\begin{equation}
    \label{eq:Wp}
    W_p(p_1, p_2) = \inf_{\gamma \in \Pi(p_1, p_2)} \left(\mathbb{E}_{(x,y)\sim\gamma} \left[ \lVert x-y\rVert^p  \right]\right)^{\frac{1}{p}},
\end{equation}
where $\Pi(p_1, p_2)$ is the set of all possible joint distributions $\gamma(x,y)$ with marginalized distributions $p_1$ and $p_2$. In 1D the Wasserstein distance has a closed form solution via Cumulative Distribution Functions (CDFs), but this evaluation is intractable in high dimension. An alternative metric, the Sliced p-Wasserstein Distance (SWD), is defined as:
\begin{equation}
    \label{eq:SWp}
    SW_p(p_1, p_2) = \left( \int_{\mathbb{S}^{d-1}} W_p^p(\mathcal{R}p_1(\cdot,\theta), \mathcal{R}p_2(\cdot,\theta)) d\theta \right)^{\frac{1}{p}},
\end{equation}
where $d\theta$ is the normalized uniform measure on $\mathbb{S}^{d-1}$. The SWD can be calculated by approximating the high dimensional integral with Monte Carlo samples. However, in high dimensions a large number of projections is required to accurately estimate SWD. This motivates to use the maximum Sliced p-Wasserstein Distance (max SWD):
\begin{equation}
    \label{eq:maxSWp}
    \max {\textrm -} SW_p(p_1, p_2) = \max_{\theta \in \mathbb{S}^{d-1}} W_p(\mathcal{R}p_1(\cdot,\theta), \mathcal{R}p_2(\cdot,\theta)),
\end{equation}
which is the maximum of the Wasserstein distance of the 1D marginalized distributions of all possible directions. SWD and max SWD are both proper distances \citep{kolouri2015radon, kolouri2019generalized}.

\section{Sliced Iterative Normalizing Flows}
\label{sec:SINF}

We consider the general problem of building a NF that maps an arbitrary PDF $p_1(x)$ to another arbitrary PDF $p_2(x)$ of the same dimensionality. We first introduce our objective function in Section \ref{subsec:max-K-SWD}. The general SINF algorithm is presented in Section \ref{subsec:algorithm}. We then consider the special cases of $p_1$ and $p_2$ being standard Normal distributions in Section \ref{subsec:SIG} and Section \ref{subsec:GIS}, respectively. In Section \ref{subsec:patch} we discuss our patch-based hierarchical strategy for modeling images.

\subsection{Maximum K-sliced Wasserstein Distance}

\label{subsec:max-K-SWD}

We generalize the idea of maximum SWD and propose maximum K-Sliced p-Wasserstein Distance (max K-SWD):
\begin{eqnarray}
    \label{eq:maxKSWp}
    \max {\textrm -} K {\textrm -} SW_p(p_1, p_2) = \max_{\{\theta_1, \cdots, \theta_K\}\  \mathrm{orthonormal}} \nonumber \\
    \left(\frac{1}{K} \sum_{k=1}^K W_p^p((\mathcal{R}p_1)(\cdot,\theta_k), (\mathcal{R}p_2)(\cdot,\theta_k))\right)^{\frac{1}{p}} .
\end{eqnarray}
In this work we fix $p=2$. The proof that max K-SWD is a proper distance is in the appendix.
If $K=1$, it becomes max SWD. For $K<d$, the idea of finding the subspace with maximum distance is similar to the subspace robust Wasserstein distance \citep{paty2019subspace}. \citet{wu2019sliced} and \citet{rowland2019orthogonal} proposed to approximate SWD with orthogonal projections, similar to max K-SWD with $K=d$. max K-SWD will be used as the objective in our proposed algorithm. It defines K orthogonal axes $\{\theta_1, \cdots, \theta_K\}$ for which the marginal distributions of $p_1$ and $p_2$ are the most different, providing a natural choice for performing 1D marginal matching in our algorithm (see Section \ref{subsec:algorithm}).

The optimization in max K-SWD is performed under the constraints that $\{\theta_1, \cdots, \theta_K\}$ are orthonormal vectors, or equivalently, $A^TA=I_K$ where $A = [\theta_1, \cdots, \theta_K]$ is the matrix whose i-th column vector is $\theta_i$. Mathematically, the set of all possible $A$ matrices is called Stiefel Manifold $V_K(\mathbb{R}^d)=\{A\in \mathbb{R}^{d\times K}: A^TA=I_K\}$, and we perform the optimization on the Stiefel Manifold following \citet{tagare2011notes}. The details of the optimization is provided in the appendix, and the procedure for estimating max K-SWD and $A$ is shown in Algorithm \ref{alg:KmaxSWD}.

\begin{algorithm}[tb]
   \caption{max K-SWD}
   \label{alg:KmaxSWD}
\begin{algorithmic}
   \STATE {\bfseries Input:} $\{x_i \sim p_1\}^N_{i=1}$, $\{y_i \sim p_2\}^N_{i=1}$, $K$, order $p$, max iteration $J_{\mathrm{maxiter}}$
   \STATE Randomly initialize $A\in V_K(\mathbb{R}^d)$
   \FOR{$j=1$ {\bfseries to} $J_{\mathrm{maxiter}}$}
   \STATE Initialize $D = 0$
   \FOR{$k=1$ {\bfseries to} $K$}
   \STATE $\theta_k = A[:,k]$ 
   \STATE Compute $\hat{x}_i=\theta_k \cdot x_i$ and $\hat{y}_i=\theta_k \cdot y_i$ for each $i$ 
   \STATE Sort $\hat{x}_i$ and $\hat{x}_j$ in ascending order s.t. $\hat{x}_{i[n]} \leq \hat{x}_{i[n+1]}$ and $\hat{y}_{j[n]} \leq \hat{y}_{j[n+1]}$
   \STATE $D = D + \frac{1}{KN}\sum_{i=1}^{N}|\hat{x}_{i[n]} - \hat{y}_{j[n]}|^p$
   \ENDFOR
   \STATE $G=[-\frac{\partial D}{\partial A_{i,j}}]$, $U=[G, A]$ , $V=[A, -G]$
   \STATE Determine learning rate $\tau$ with backtracking line search $A = A - \tau U (I_{2K}+\frac{\tau}{2}V^TU)^{-1}V^TA$ 
   \IF{$A$ has converged}
   \STATE Early stop
   \ENDIF
   \ENDFOR
   \STATE {\bfseries Output:} $D^{\frac{1}{p}} \approx \max {\textrm -} K {\textrm -} SW_p$, $A \approx [\theta_1, \cdots, \theta_K]$
\end{algorithmic}
\end{algorithm}

\subsection{Proposed SINF Algorithm}

\label{subsec:algorithm}
The proposed SINF algorithm is based on iteratively minimizing the max K-SWD between the two distributions $p_1$ and $p_2$. Specifically, SINF iteratively solves for the orthogonal axes where the marginals of $p_1$ and $p_2$ are most
different (defined by max K-SWD), and then match the 1D marginalized distribution of $p_1$ to $p_2$ on those axes. This is motivated by the inverse Radon Transform (Equation \ref{eq:inverseRadon}) and Cram{\'e}r-Wold theorem, which suggest that matching the high dimensional distributions is equivalent to matching the 1D slices on all possible directions, decomposing the high dimensional problem into a series of 1D problems. See Figure \ref{fig:SINF} for an illustration of the SINF algorithm. Given a set of i.i.d. samples $X$ drawn from $p_1$, in each iteration, a set of 1D marginal transformations $\{\Psi_k\}_{k=1}^{K}$ ($K \leq d$ where $d$ is the dimensionality of the dataset) are applied to the samples on orthogonal axes $\{\theta_k\}_{k=1}^K$ to match the 1D marginalized PDF of $p_2$ along those axes. Let $A = [\theta_1, \cdots, \theta_K]$ be the matrix derived from max K-SWD optimization (algorithm \ref{alg:KmaxSWD}), that contains $K$ orthogonal axes($A^TA=I_K$). Then the transformation at iteration $l$ of samples $X_l$ can be written as
\footnote{{\bf Notation definition}: In this paper we use $l$, $k$, $j$ and $m$ to represent different iterations of the algorithm, different axes $\theta_k$, different gradient descent iterations of max K-SWD calculation (see Algorithm \ref{alg:KmaxSWD}), and different knots in the spline functions of 1D transformation, respectively.}
\begin{equation}
    \label{eq:forward}
    X_{l+1} = A_l \mathbf{\Psi}_l(A_l^T X_l)+ X_l^{\perp} ,
\end{equation}
where $X_l^{\perp} = X_l - A_lA_l^TX_l$ contains the components that are perpendicular to $\theta_1, ..., \theta_K$ and is unchanged in iteration $l$. $\mathbf{\Psi}_l=[\Psi_{l1}, \cdots, \Psi_{lK}]^T$ is the marginal mapping of each dimension of $A_l^T X_l$, and its components are required to be monotonic and differentiable. The transformation of Equation \ref{eq:forward} can be easily inverted:
\begin{equation}
    \label{eq:inverse}
    X_l = A_l \mathbf{\Psi}_l^{-1}(A_l^T X_{l+1})+ X_l^{\perp} ,
\end{equation}
where $X_l^{\perp} = X_l - A_lA_l^TX_l =  X_{l+1} - A_lA_l^TX_{l+1}$. The Jacobian determinant of the transformation is also efficient to calculate (see appendix for the proof):
\begin{equation}
    \label{eq:jacobian}
    \det(\frac{\partial X_{l+1}}{\partial X_l}) = \prod_{k=1}^K \frac{d\Psi_{lk}(x)}{dx} .
\end{equation}


\begin{figure}[t]
     \centering
      \includegraphics[width=\linewidth]{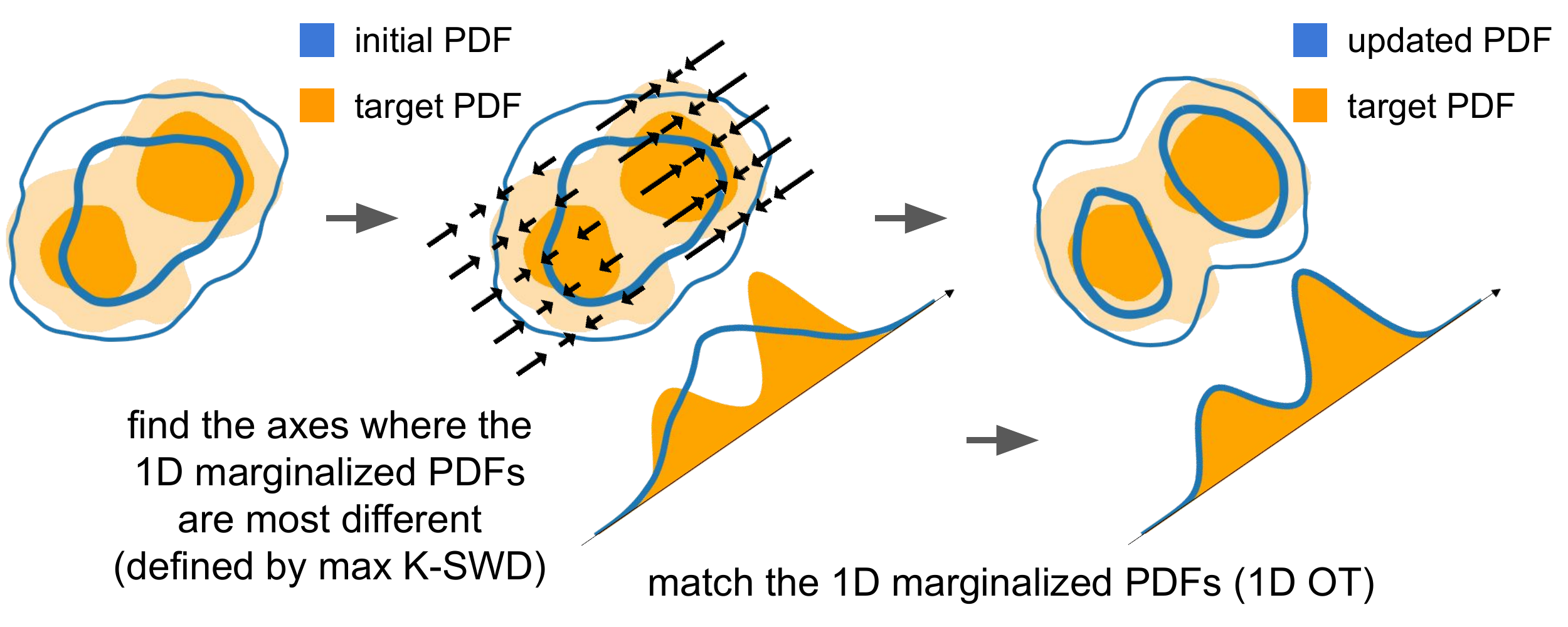}
     \caption{Illustration of 1 iteration of SINF algorithm with $K=1$.}
     \label{fig:SINF}
     \vskip -0.1in
\end{figure}

At iteration $l$, the SINF objective can be written as:
\begin{eqnarray}
    \label{eq:minimax}
    \mathcal{F}_l = \min_{\{\Psi_{l1}, \cdots, \Psi_{lK}\}}\ \max_{\{\theta_{l1}, \cdots, \theta_{lK}\}\  \mathrm{orthonormal}} \nonumber \\
    \left(\frac{1}{K} \sum_{k=1}^K W_p^p(\Psi_{lk}((\mathcal{R}p_{1,l})(\cdot,\theta_{lk})), (\mathcal{R}p_2)(\cdot,\theta_{lk}))\right)^{\frac{1}{p}} .
\end{eqnarray}
The algorithm first optimizes $\theta_{lk}$ to maximize the objective, with $\Psi_{lk}$ fixed to identical transformations (equivalent to Equation \ref{eq:maxKSWp}). Then the axes $\theta_{lk}$ are fixed and the objective is minimized with  marginal matching $\Psi_l$. The samples are updated, and the process is repeated until convergence.  

Let $p_{1,l}$ be the transformed $p_1$ at iteration $l$. The $k$th component of $\mathbf{\Psi}_l$, $\Psi_{l,k}$, maps the 1D marginalized PDF of $p_{1,l}$ to $p_2$ and has an OT solution:
\begin{equation}
    \label{eq:1D}
    \Psi_{l,k}(x) = F_k^{-1}(G_{l,k}(x)) ,
\end{equation}
where $G_{l,k}(x)=\int_{-\infty}^x (\mathcal{R}p_{1,l})(t,\theta_k)dt$ and $F_k(x)=\int_{-\infty}^x (\mathcal{R}p_2)(t,\theta_k)dt$ are the CDFs of $p_{1,l}$ and $p_2$ on axis $\theta_k$, respectively. The CDFs can be estimated using the quantiles of the samples (in SIG Section \ref{subsec:SIG}), or using Kernel Density Estimation (KDE, in GIS Section \ref{subsec:GIS}). Equation \ref{eq:1D} is monotonic and invertible. We choose to parametrize it with monotonic rational quadratic splines \citep{gregory1982piecewise, durkan2019neural}, which are continuously-differentiable and allows analytic inverse.  More details about the spline procedure are given in the appendix. We summarize SINF in Algorithm \ref{alg:NF}.

\begin{algorithm}[tb]
   \caption{Sliced Iterative Normalizing Flow}
   \label{alg:NF}
\begin{algorithmic}
   \STATE {\bfseries Input:} $\{x_i \sim p_1\}^N_{i=1}$, $\{y_i \sim p_2\}^N_{i=1}$, $K$, number of iteration $L_{\mathrm{iter}}$
   \FOR{$l=1$ {\bfseries to} $L_{\mathrm{iter}}$}
   \STATE $A_l = \textrm{max K-SWD}(x_i, y_i, K)$
   \FOR{$k=1$ {\bfseries to} $K$}
   \STATE $\theta_k = A_l[:,k]$ 
   \STATE Compute $\hat{x}_i=\theta_k \cdot x_i$ and $\hat{y}_i=\theta_k \cdot y_i$ for each $i$ 
   \STATE $\tilde{x}_m = \textrm{quantiles} (\textrm{PDF}(\hat{x}_{i}))$\\ $\tilde{y}_m = \textrm{quantiles} (\textrm{PDF}(\hat{y}_{i}))$
   \STATE $\psi_{l,k} = \textrm{RationalQuadraticSpline}(\tilde{x}_m, \tilde{y}_m)$
   \ENDFOR
   \STATE $\mathbf{\Psi}_l=[\Psi_{l1}, \cdots, \Psi_{lK}]$
   \STATE Update $x_i = x_i - A_lA_l^Tx_i + A_l \mathbf{\Psi}_l(A_l^T x_i)$ 
   \ENDFOR
\end{algorithmic}
\end{algorithm}

The proposed algorithm iteratively minimizes the max K-SWD between the transformed $p_1$ and $p_2$. The orthonomal vectors $\{\theta_1,\cdots,\theta_K\}$ specify $K$ axes along which the marginalized PDF between $p_{1,l}$ and $p_2$ are most different, thus maximizing the gain at each iteration and improving the efficiency of the algorithm. In the appendix we show empirically that the model is able to converge with two orders of magnitude fewer iterations than random axes, and it also leads to better sample quality. This is because as the dimensionality $d$ grows, the number of slices $(\mathcal{R}p)(\cdot,\theta)$ required to approximate $p(x)$ using inverse Radon formula scales as $L^{d-1}$ \citep{kolouri2015radon}, where $L$ is the number of slices needed to approximate a similar smooth 2D distribution. Therefore, if $\theta$ are randomly chosen, it takes a large number of iterations to converge in high dimensions due to the curse of dimensionality. Our objective function reduces the curse of dimensionality in high dimensions by identifying the most relevant directions first. 

$K$ is a free hyperparameter in our model. In the appendix we show empirically that the convergence of the algorithm is insensitive to the choice of $K$, and mostly depends on the total number of 1D transformations $L_{\mathrm{iter}}\times K$. 

Unlike KL-divergence, which is invariant under the flow transformations, max K-SWD is different in data space and in latent space. Therefore the direction of building the flow model is of key importance. In the next two sections we discuss two different ways of building the flow, which are good at sample generation and density estimation, respectively.

\subsection{Sliced Iterative Generator (SIG)}

\label{subsec:SIG}

For Sliced Iterative Generator (SIG) $p_1$ is a standard Normal distribution, and $p_2$ is the target distribution. The model iteratively maps the Normal distribution to the target distribution using 1D slice transformations. 
SIG directly minimizes the max K-SWD between the generated distribution and the target distribution, and is able to generate high quality samples. The properties of SIG are summarized in Table \ref{tab:comparison}.

Specifically, one first draws a set of samples from the standard Normal distribution, and then iteratively updates the samples following Equation \ref{eq:forward}. Note that in the NF framework, Equation \ref{eq:forward} is the inverse of transformation $f_l$ in Equation \ref{eq:flow}. The $\Psi$ transformation and the weight matrix $A$ are learned using Equation \ref{eq:1D} and Algorithm \ref{alg:KmaxSWD}. In Equation \ref{eq:1D} we estimate the CDFs using the quantiles of the samples.

\subsection{Gaussianizing Iterative Slicing (GIS)}

\label{subsec:GIS}

For Gaussianizing Iterative Slicing (GIS)
$p_1$ is the target distribution and $p_2$ is a standard Normal distribution. The model iteratively gaussianizes the target distribution, and the mapping is learned in the reverse direction of SIG. 
In GIS the max K-SWD between latent data and the Normal distribution is minimized, thus the model performs well in density estimation, even though its learning objective is not $\log p$. The comparison between SIG and GIS is shown in Table \ref{tab:comparison}.

We add regularization to GIS for density estimation tasks to further improve the performance and reduce overfitting. The regularization is added in the following two aspects:\newline
1) The weight matrix $A_l$ is regularized by limiting the maximum number of iterations $J_{\mathrm{maxiter}}$ (see Algorithm \ref{alg:KmaxSWD}). We set $J_{\mathrm{maxiter}} = N/d$. 
Thus for very small datasets ($N/d \to 1$) the axes of marginal transformation are almost random. This has no effect on datasets of regular size. \newline
2) The CDFs in Equation \ref{eq:1D} are estimated using KDE, and the 1D marginal transformation is regularized with:
\begin{equation}
    \label{eq:alpha}
    \tilde{\psi}_{l,k}(x) =  (1-\alpha)\psi_{l,k}(x) + \alpha x ,
\end{equation}
where $\alpha \in [0,1)$ is the regularization parameter, and $\tilde{\psi}_{l,k}$ is the regularized transformation. In the appendix we show that as $\alpha$ increases, the performance improves, but more iterations are needed to converge. Thus $\alpha$ controls the trade-off between performance and speed.




\begin{table}[tb]
  \caption{Comparison between SIG and GIS}
  \label{tab:comparison}
  \vskip 0.15in
  \centering
  \begin{tabular}{>{\centering}c|>{\centering}c|>{\centering\arraybackslash}c}
    \toprule
     Model & SIG & GIS\\ 
    \midrule\midrule
    Initial PDF $p_1$ & Gaussian & $p_{\mathrm{data}}$\\
    Final PDF $p_2$ & $p_{\mathrm{data}}$ & Gaussian \\
    \multirow{2}{*}{Training} & Iteratively maps & Iteratively maps \\
    & Gaussian to $p_{\mathrm{data}}$ & $p_{\mathrm{data}}$ to Gaussian\\
    NF structure & Yes & Yes \\
    \multirow{2}{*}{Advantage} & \multirow{2}{*}{Good samples} & Good density\\
    & & estimation\\
    \bottomrule
  \end{tabular}
  \vskip -0.1in
\end{table}

\subsection{Patch-Based Hierarchical Approach}

\label{subsec:patch}

Generally speaking, the neighboring pixels in images have stronger correlations than pixels that are far apart. This fact has been taken advantage by convolutional neural networks, which outperform Fully Connected Neural Networks (FCNNs) and have become standard building blocks in computer vision tasks. Like FCNNs, vanilla SIG and GIS make no assumption about the structure of the data and cannot model high dimensional images very well. \citet{meng2020gaussianization} propose a patch-based approach, 
which decomposes an $S \times S$ image into $p\times p$ patches, with $q\times q$ neighboring pixels in each patch ($S=pq$). In each iteration the marginalized distribution of each patch is modeled separately without considering the correlations between different patches. This approach effectively reduces the dimensionality from $S^2$ to $q^2$, at the cost of ignoring the long range correlations. Figure \ref{fig:patch} shows an illustration of the patch-based approach. 

To reduce the effects of ignoring the long range correlations, we propose a hierarchical model. In SIG, we start from modeling the entire images, which corresponds to $q=S$ and $p=1$. After some iterations the samples show correct structures, indicating the long range correlations have been modeled well. We then gradually decrease the patch size $q$ until $q=2$, which allows us to gradually focus on the smaller scales. Assuming a periodic boundary condition, we let the patches randomly shift in each iteration. If the patch size $q$ does not divide $S$, we set $p=\lfloor S/q\rfloor$ and the rest of the pixels are kept unchanged.

\begin{figure}[t]
     \centering
      \includegraphics[width=\linewidth]{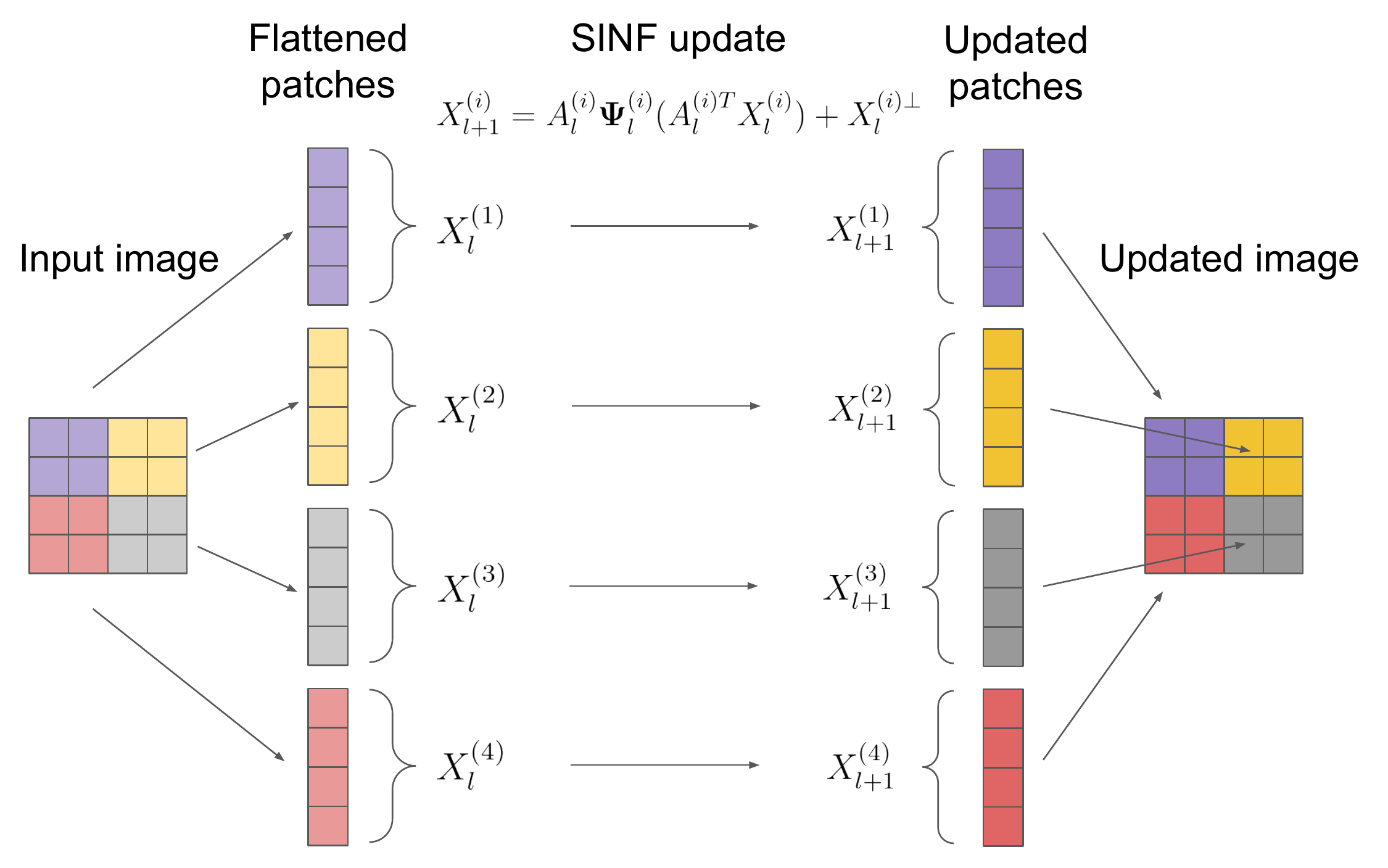}
     \caption{Illustration of the patch-based approach with $S=4$, $p=2$ and $q=2$. At each iteration, different patches are modeled separately. The patches are randomly shifted in each iteration assuming periodic boundaries.}
     \label{fig:patch}
     \vskip -0.1in
\end{figure}

\section{Related Work}

\label{sec:related}

Iterative normalizing flow models called 
RBIG \citep{chen2001gaussianization, laparra2011iterative} are simplified versions of GIS, as they are based on a succession of rotations followed by 1D marginal Gaussianizations. Iterative Distribution Transfer (IDT) \citep{pitie2007automated} is a similar algorithm but does not require the base distribution to be a Gaussian. These models do not scale well to high dimensions because they do not have a good way of choosing the axes 
(slice directions), and they are not competitive against modern NFs trained on $p(x)$ \citep{meng2020gaussianization}. 
\citet{meng2019ppmm} use a similar algorithm called Projection Pursuit Monge Map (PPMM) to construct OT maps. They propose to find the most informative axis using Projection Pursuit (PP) \cite{FreedmanPP} at each iteration, and show that PPMM works well in low-dimensional bottleneck settings ($d=8$). PP scales as $\mathcal{O}(d^3)$, which makes PPMM scaling to high dimensions prohibitive. A DL, non-iterative version of these models is Gaussianization Flow (GF) \cite{meng2020gaussianization}, which trains on $p(x)$ and achieves good density estimation results in low dimensions, but does not have good sampling properties in high dimensions. RBIG, GIS and GF have similar architectures but are trained differently. We compare their density estimation results in Section \ref{subsec:density}.

Another iterative generative model is Sliced Wasserstein Flow (SWF) \citep{liutkus2018sliced}. Similar to SIG, SWF tries to minimize the SWD between the distributions of samples and the data, and transforms this problem into solving a d dimensional PDE. The PDE is solved iteratively by doing a gradient flow in the Wasserstein space, and works well for low dimensional bottleneck features. However, in each iteration the algorithm requires evaluating an integral over the $d$ dimensional unit sphere approximated with Monte Carlo integration, which does not scale well to high dimensions. Another difference with SIG is that SWF does not have a flow structure, cannot be inverted, and does
not provide the likelihood. We compare the sample qualities between SWF and SIG in Section \ref{subsec:samples}.

SWD, max SWD and other slice-based distance 
(e.g. Cram{\'e}r-Wold distance) 
have been widely used in training generative models \citep{deshpande2018generative, deshpande2019max, wu2019sliced, kolouri2018sliced, knop2018cramer, nguyen2020improving, nguyen2020distributional, nadjahi2020statistical}. \citet{wu2019sliced} propose a differentiable SWD block composed of a rotation followed by marginalized Gaussianizations, but unlike RBIG, the rotation matrix is trained in an end-to-end DL fashion. They propose Sliced Wasserstein AutoEncoder (SWAE) by adding SWD blocks to an AE to regularize 
the latent variables, and show that its sample quality outperforms VAE and AE + RBIG. \citet{nguyen2020improving, nguyen2020distributional} generalize the max-sliced approach using parametrized distributions over projection axes. \citet{nguyen2020improving} propose Mixture Spherical Sliced Fused Gromov Wasserstein (MSSFG), 
which samples the slice axes around a few informative directions following Von Mises-Fisher distribution. They apply MSSFG to training of Deterministic Relational regularized AutoEncoder (DRAE) and name it mixture spherical DRAE (ms-DRAE). \citet{nguyen2020distributional} go further and propose Distributional Sliced Wasserstein distance (DSW), which tries to find the optimal axes distribution by parametrizing it with a neural network. They apply DSW to the training of GANs, and we will refer to their model as DSWGAN in this paper. We compare the sample qualities between SIG, SWAE, ms-DRAE, DSWGAN and other similar models in Section \ref{subsec:samples}.

\citet{grover2018flow} propose Flow-GAN using a NF as the generator of a GAN, so the model can perform likelihood evaluation, and allows both maximum likelihood and adversarial training. Similar to our work they find that adversarial training gives good samples but poor $ p(x)$, while training by maximum likelihood results in bad samples. Similar to SIG, the adversarial version of Flow-GAN minimizes the Wasserstein distance between samples and data, and has a NF structure. We compare their samples in Section \ref{subsec:samples}.

\begin{table*}[htb]
  \caption{Negative test log-likelihood for tabular datasets measured in nats, and image datasets measured in bits/dim (lower is better). }
  \label{tab:density}
  \vskip 0.15in
  \centering
  \begin{tabular}{>{\centering}c|>{\centering}c|>{\centering}c>{\centering}c>{\centering}c>{\centering}c>{\centering}c|>{\centering}c>{\centering\arraybackslash}c}
    \toprule
    & Method & POWER & GAS & HEPMASS & MINIBOONE & BSDS300 & MNIST & Fashion\\
    \midrule\midrule
    \multirow{2}{*}{iterative}
    & RBIG & 1.02 & 0.05 & 24.59 & 25.41 & -115.96 & 1.71 & 4.46\\
    & GIS (this work) & -0.32 & -10.30 & 19.00 & 14.26 & -155.75 & 1.34 & 3.22\\
    \midrule
    \multirow{6}{*}{\shortstack{maximum\\likelihood}}
    & GF & -0.57 & -10.13 & 17.59 & 10.32 & -152.82 & 1.29 & 3.35 \\
    & Real NVP & -0.17 & -8.33 & 18.71 & 13.55 & -153.28 & 1.06 & 2.85\\
    & Glow & -0.17 & -8.15 & 18.92 & 11.35 & -155.07 & 1.05 & 2.95\\
    & FFJORD & -0.46 & -8.59 & 14.92 & 10.43 & -157.40 & 0.99 & - \\
    & MAF & -0.30 & -10.08 & 17.39 & 11.68 & -156.36 & 1.89 & -\\
    & RQ-NSF (AR) & -0.66 & -13.09 & 14.01 & 9.22 & -157.31 & - & -\\
    \bottomrule
  \end{tabular}
  \vskip -0.1in
\end{table*}

\section{Experiments}

\subsection{Density Estimation $p(x)$ of Tabular Datasets}
\label{subsec:density}

\begin{figure}[t]
     \centering
     \begin{subfigure}[t]{0.49\linewidth}
         \includegraphics[width=\textwidth]{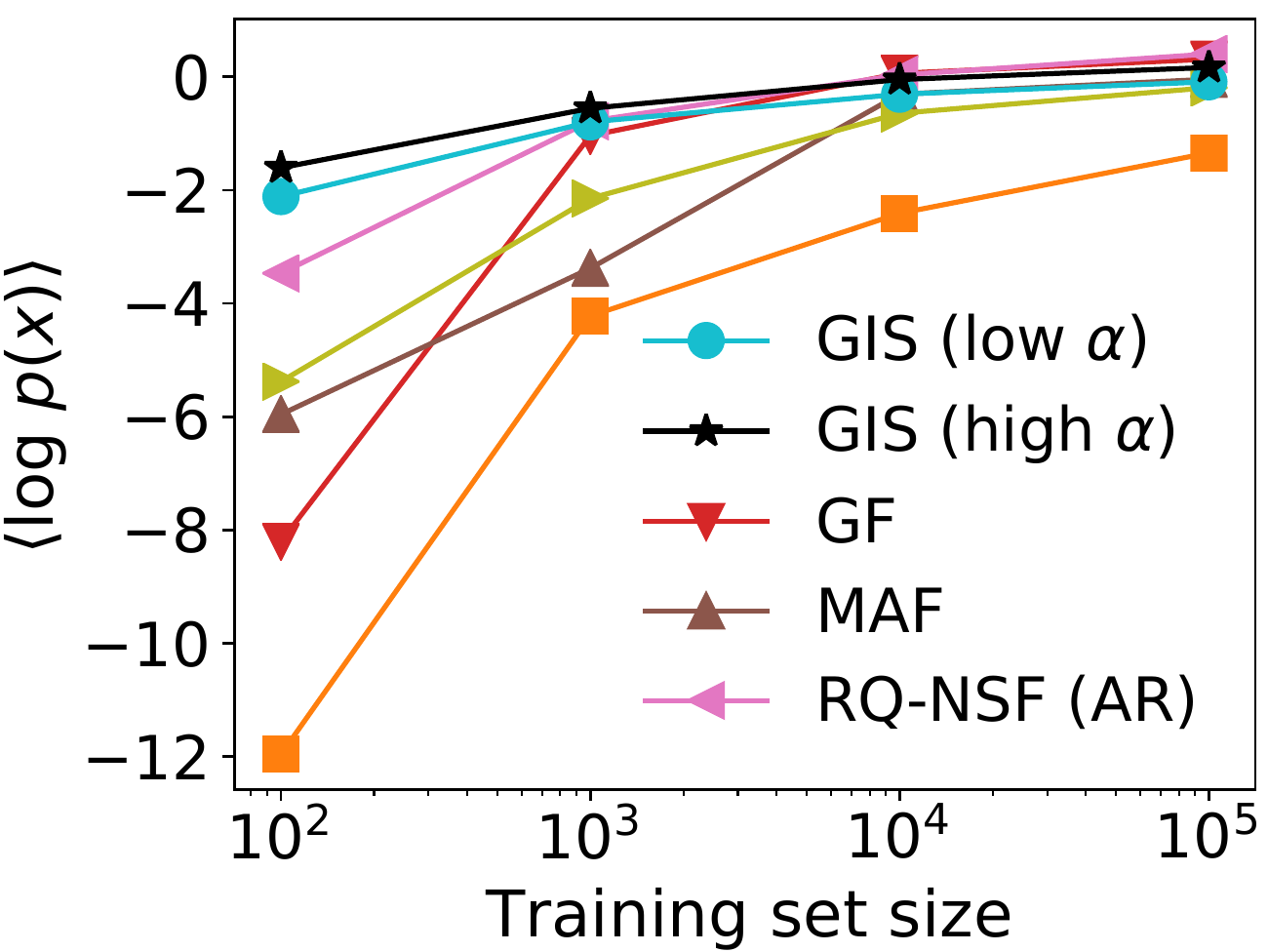}
         \caption{POWER (6D)}
     \end{subfigure}
     \begin{subfigure}[t]{0.49\linewidth}
         \includegraphics[width=\textwidth]{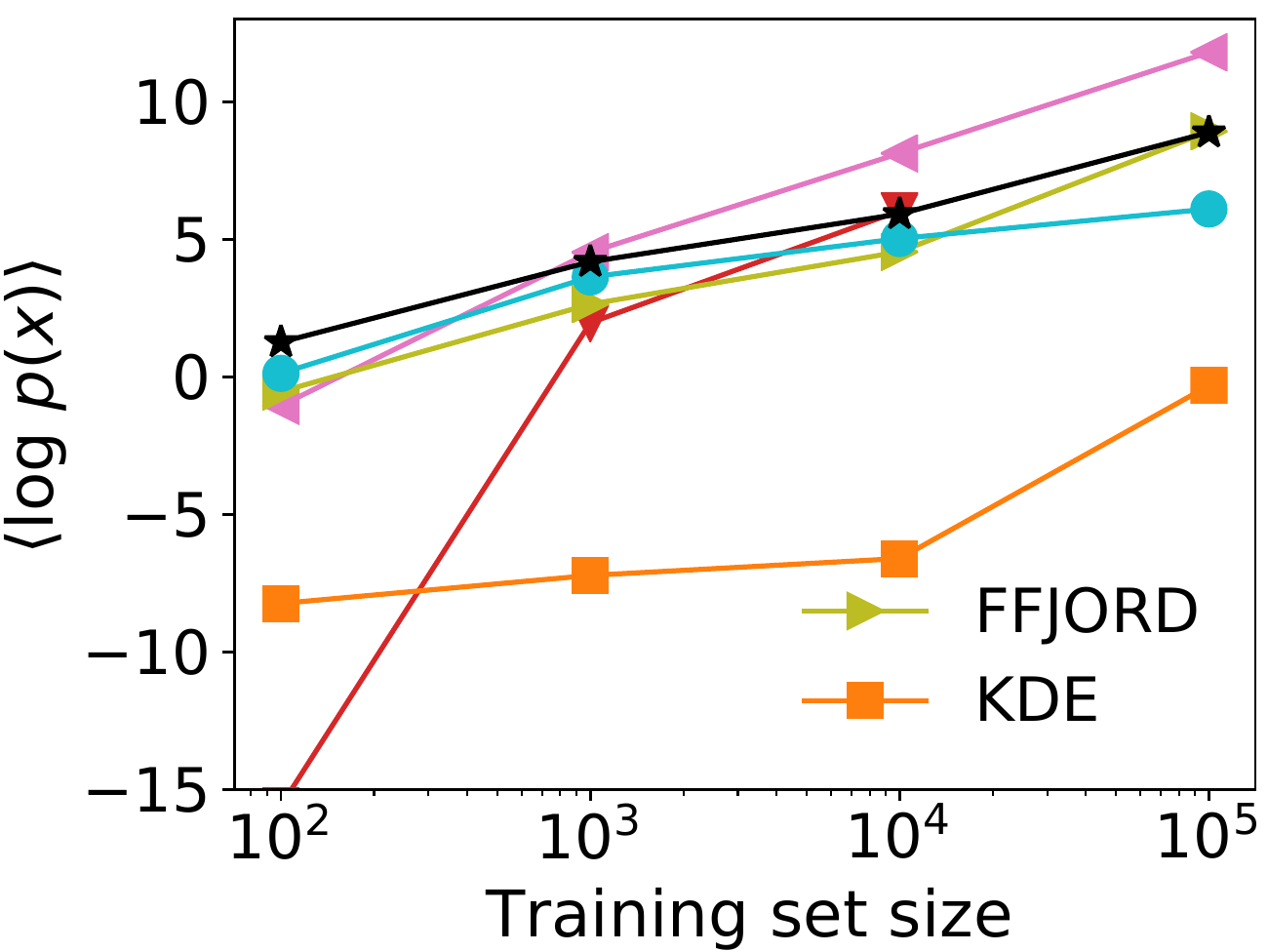}
         \caption{GAS (8D)}
     \end{subfigure}
     \begin{subfigure}[t]{0.49\linewidth}
         \includegraphics[width=\textwidth]{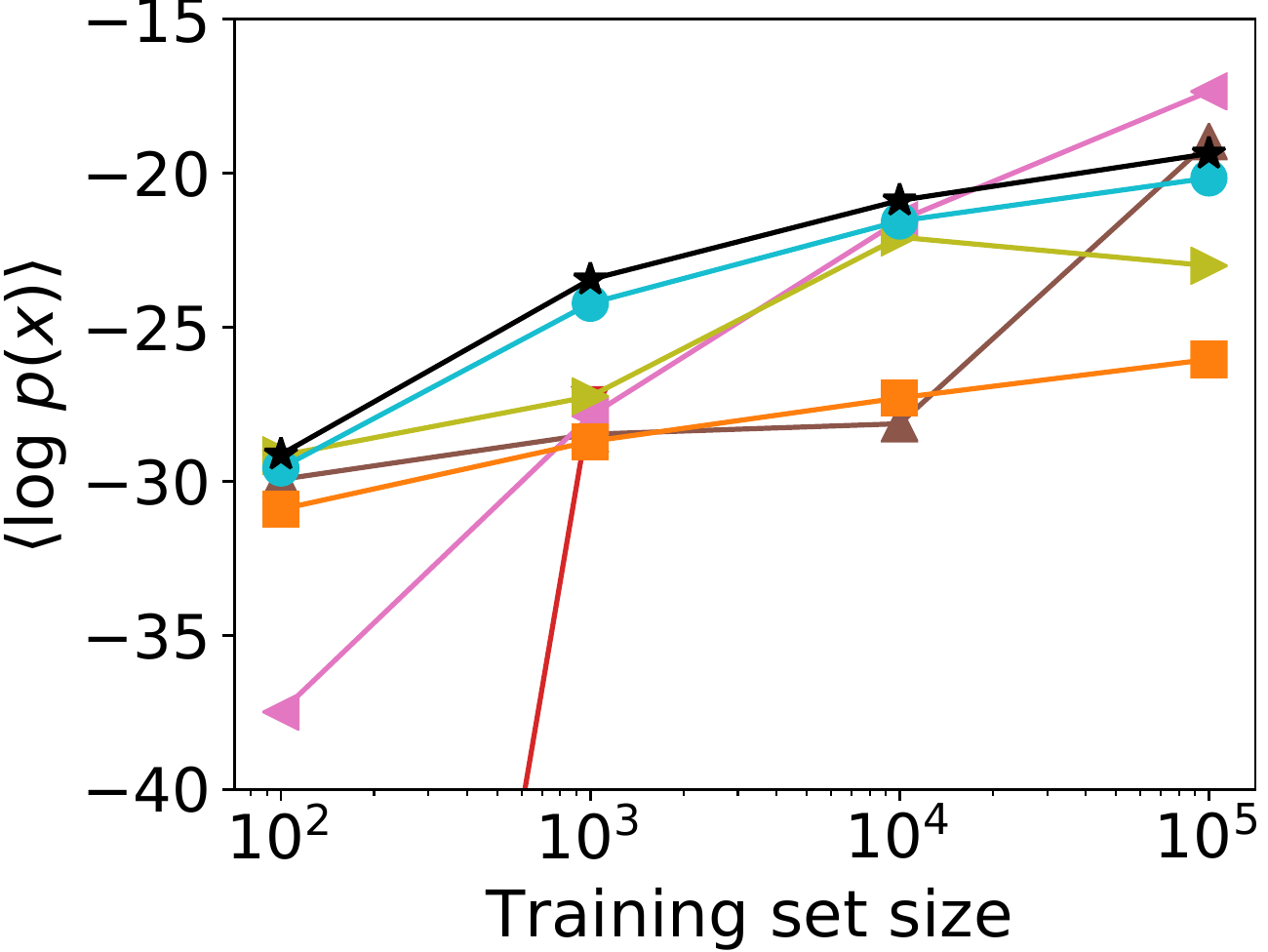}
         \caption{HEPMASS (21D)}
     \end{subfigure}
     \begin{subfigure}[t]{0.49\linewidth}
         \includegraphics[width=\textwidth]{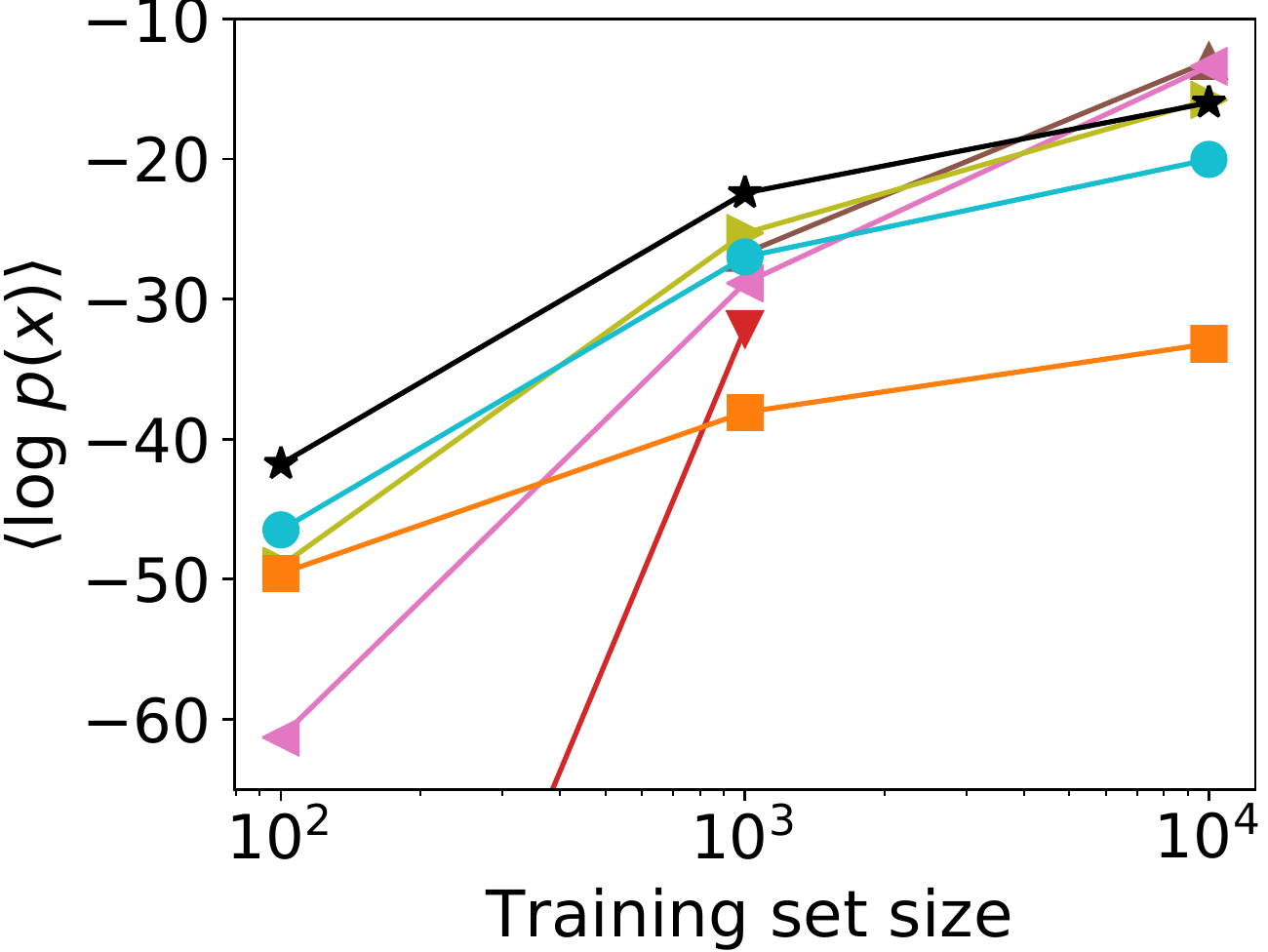}
         \caption{MINIBOONE (43D)}
     \end{subfigure}
     \begin{subfigure}[t]{0.49\linewidth}
         \includegraphics[width=\textwidth]{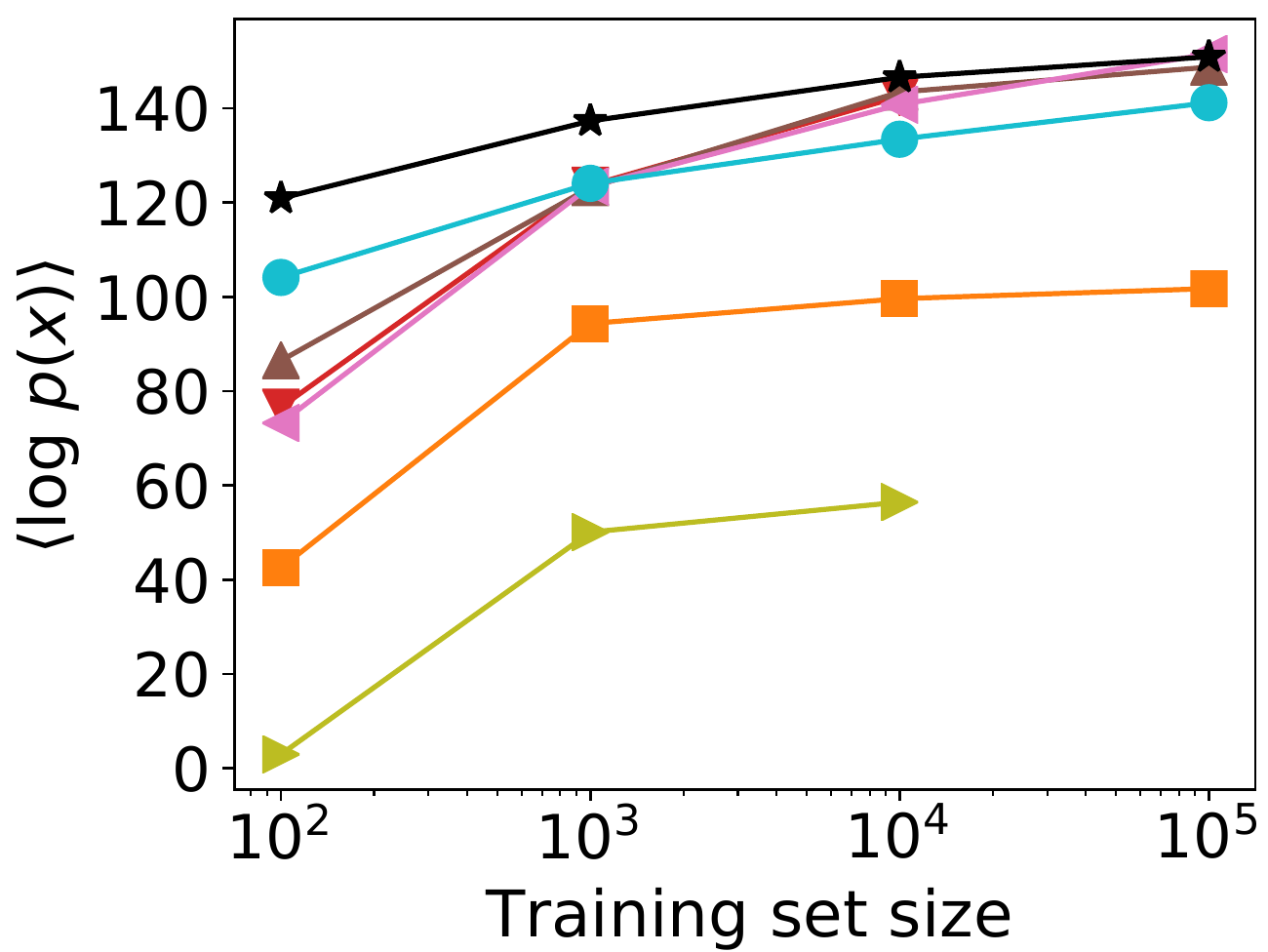}
         \caption{BSDS300 (63D)}
     \end{subfigure}
        \caption{Density estimation on small training sets. The legends in panel (a) and (b) apply to other panels as well. At 100 training data GIS 
        has the best performance in all cases.
        }
        \label{fig:density}
      \vskip -0.1in
\end{figure}

\begin{table}[htb]
  \caption{Averaged training time of different NF models on small datasets ($N_{\mathrm{train}}=100$) measured in seconds. All the models are tested on both a cpu and a K80 gpu, and the faster results are reported here (the results with * are run on gpus.). P: POWER, G: GAS, H: HEPMASS, M: MINIBOONE, B: BSDS300.}
  \label{tab:time}
  \vskip 0.15in
  \centering
  \begin{threeparttable}
  \begin{tabular}{>{\centering}c|>{\centering}c>{\centering}c>{\centering}c>{\centering}c>{\centering\arraybackslash}c}
    \toprule
    Method & P & G & H & M & B\\
    \midrule\midrule
    GIS (low $\alpha$) & $0.53$ & $1.0$ & $0.63$ & $3.5$ & $7.4$  \\
    GIS (high $\alpha$) & $6.8$ & $9.4$ & $7.3$ & $44.1$ & $69.1$\\
    GF & $113^*$ & $539^*$ & $360^*$ & $375^*$ & $122^*$ \\
    MAF & $18.4$ & -\tnote{1} & $10.2$ & -\tnote{1} & $32.1$ \\
    FFJORD & $1051$ & $1622$ & $1596$ & $499^*$ & $4548^*$ \\
    RQ-NSF (AR) & $118$ & $127$ & $55.5$ & $38.9$ & $391$\\
    \bottomrule
  \end{tabular}
  \begin{tablenotes}
  \item[1] Training failures.
  \end{tablenotes}
  \end{threeparttable}
  \vskip -0.1in
\end{table}

We perform density estimation with GIS on four UCI datasets \citep{lichman2013uci} and BSDS300  \citep{martin2001database}, as well as image datasets MNIST \citep{lecun1998gradient} and Fashion-MNIST \citep{xiao2017online}. The data preprocessing of UCI datasets and BSDS300 follows \citet{papamakarios2017masked}. In Table \ref{tab:density} we compare our results with RBIG \citep{laparra2011iterative} and GF \citep{meng2020gaussianization}. The former can be seen as GIS with random axes to apply 1D gaussianization, while the latter can be seen as training non-iterative GIS with MLE training on $p(x)$. We also list other NF models Real NVP \citep{dinh2016density}, Glow \citep{kingma2018glow}, FFJORD \citep{grathwohl2018ffjord}, MAF \citep{papamakarios2017masked} and RQ-NSF (AR)\citep{durkan2019neural} for comparison. 

We observe that RBIG performs significantly worse than current SOTA. GIS outperforms RBIG and is the first iterative algorithm that achieves comparable performance compared to maximum likelihood models. This is even more impressive given that GIS is not trained on $p(x)$, yet it outperforms GF on $p(x)$ on GAS, BSDS300 and Fashion-MNIST.

The transformation at each iteration of GIS is well defined, and the algorithm is very stable even for small training sets. To test the stability and performance we compare the density estimation results with other methods 
varying the size of the training set $N$ 
(from $10^2$ to $10^5$). 
For GIS we consider two hyperparameter settings: large regularization $\alpha$ (Equation \ref{eq:alpha}) for better $\log p$ performance, and small regularization $\alpha$ for faster training. For other NFs we use settings recommended by their original paper, and set the batch size to $\min(N/10, N_{\mathrm{batch}})$, 
where $N_{\mathrm{batch}}$ is the batch size suggested by the original paper. All the models are trained until the validation $\log p_{\mathrm{val}}$ stops improving, and for KDE the kernel width is chosen to maximize $\log p_{\mathrm{val}}$. Some non-GIS NF models diverged during training or used more memory than our GPU, and are not shown in the plot. The results in Figure \ref{fig:density} show that GIS is more stable compared to other NFs and outperforms them on small training sets. This highlights that GIS is less sensitive to hyper-parameter optimization and achieves good performance out of the box. GIS training time varies with data size, but is generally lower than other NFs for small training sets. We report the training time for $100$ training data in Table \ref{tab:time}. GIS with small regularization $\alpha$ requires significantly less time than other NFs, while
still outperforming them at 100 training size.

\subsection{Generative Modeling of Images}

\label{subsec:samples}

\begin{table*}[htb]
  \caption{FID scores on different datasets (lower is better). The errors are generally smaller than the differences.}
  \label{tab:FID}
  \vskip 0.15in
  \centering
  \begin{threeparttable}
  \begin{tabular}{>{\centering}c|>{\centering}c|>{\centering}c>{\centering}c>{\centering}c>{\centering\arraybackslash}c}
    \toprule
    & Method & MNIST & Fashion & CIFAR-10 & CelebA\\ 
    \midrule\midrule
    \multirow{2}{*}{iterative}
    & SWF & $225.1$ & $207.6$ & - & -\\
    & SIG ($T=1$) (this work) & $\mathbf{4.5}$ & $\mathbf{13.7}$ & $66.5$ & $37.3$\\\midrule
    \multirow{5}{*}{\shortstack{adversarial\\training}}
    & Flow-GAN (ADV) & $155.6$ & $216.9$ & $71.1$ & -\\
    & DSWGAN & - & - & $56.4$ & $66.9$\\
    & WGAN & $6.7$ & $21.5$ & $\mathbf{55.2}$ & $41.3$\\
    & WGAN GP & $20.3$ & $24.5$ & $55.8$ & $\mathbf{30.0}$\\
    & Best default GAN & $\sim 10$ & $\sim 32$ & $\sim 70$ & $\sim 48$ \\\midrule
    \multirow{6}{*}{\shortstack{AE based}}
    & SWAE\citep{wu2019sliced} & - & - & $107.9$ & $48.9$\\
    & SWAE\citep{kolouri2018sliced} & $29.8$ & $74.3$ & $141.9$ & $53.9$\\
    & CWAE & $23.6$ & $57.1$ & $120.0$ & $49.7$\\
    & ms-DRAE & $43.6$ & - & - & $46.0$\\
    & PAE & - & $28.0$ & - & $49.2$\\
    & two-stage VAE & $12.6$ & $29.3$ & $96.1$
    & $44.4$\\
    \bottomrule
  \end{tabular}
  \end{threeparttable}
  \vskip -0.1in
\end{table*}


We evaluate SIG as a generative model of images using the following 4 datasets: MNIST, Fashion-MNIST, CIFAR-10 \citep{krizhevsky2009learning} and Celeb-A (cropped and interpolated to $64\times 64$ resolution) \citep{liu2015faceattributes}.
In Figure \ref{fig:sample} we show samples of these four datasets. For MNIST, Fashion-MNIST and CelebA dataset we show samples from the model with reduced temperature $T=0.85$ (i.e., sampling from a Gaussian distribution with standard deviation $T=0.85$ in latent space), which slightly improves the sample quality \citep{parmar2018image, kingma2018glow}. We report the final FID score (calculated using temperature T=1) in Table \ref{tab:FID}, where we compare our results with similar algorithms SWF and Flow-Gan (ADV). We also list the FID scores of some other generative models for comparison, including models using slice-based distance SWAEs (two different models with the same name) \citep{wu2019sliced, kolouri2018sliced}, Cramer-Wold AutoEncoder (CWAE) \citep{knop2018cramer}, ms-DRAE \citep{nguyen2020improving} and DSWGAN \citep{nguyen2020distributional}, Wasserstein GAN models \citep{arjovsky2017wasserstein, gulrajani2017improved}, and other GANs and AE-based models Probablistic AutoEncoder (PAE) \citep{bohm2020probabilistic} and two-stage VAE \citep{dai2019diagnosing,xiao2019generative}.
The scores of WGAN and WGAN-GP models are taken from \citet{lucic2018gans}, who performed a large-scale testing protocol over different GAN models. The "Best default GAN" is extracted from Figure 4 of \citet{lucic2018gans}, indicating the lowest FID scores from different GAN models with the hyperparameters suggested by original authors. 
Vanilla VAEs generally do not perform as well as two stage VAE, and thus are not shown in the table.
NF models usually do not report FID scores. PAE combines AEs with NFs and we expect it to outperform most NF models in terms of sample quality due to the use of AEs. We notice that previous iterative algorithms are unable to produce good samples on high dimensional image datasets (see Table \ref{tab:FID} and Figure \ref{fig:improve} for SWF samples; see Figure 6 and 7 of \citet{meng2020gaussianization} for RBIG samples). However, SIG obtains the best FID scores on MNIST and Fashion-MNIST, while on CIFAR-10 and CelebA it also outperforms similar algorithms and AE-based models, and gets comparable results to GANs.
In Figure \ref{fig:iteration} we show samples at different iterations. 
In Figure \ref{fig:interpolation} we display interpolations between SIG samples, and the nearest training data, to verify we are not memorizing the training data.

\begin{figure}
     \centering
     \begin{subfigure}[]{\linewidth}
         \centering \includegraphics[width=\linewidth]{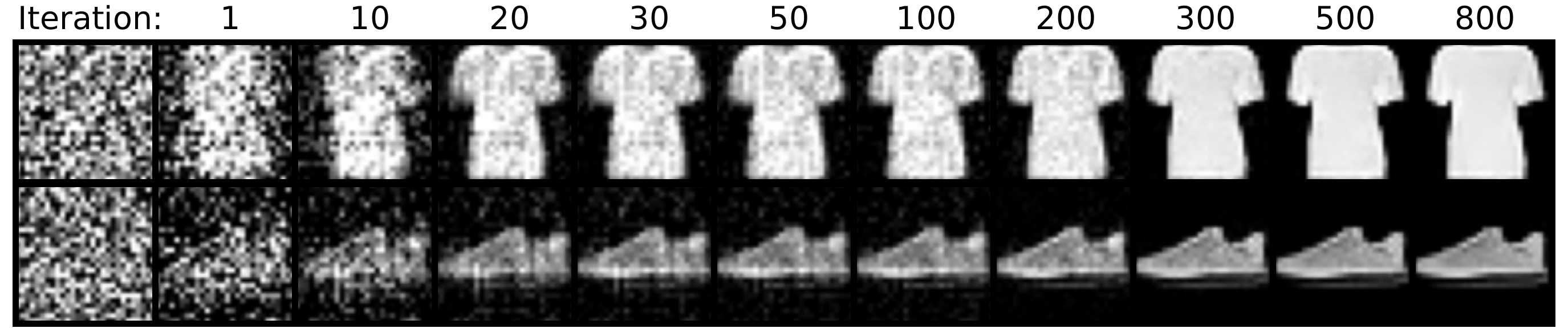}
     \end{subfigure}
     \hfill
     \begin{subfigure}[]{\linewidth}
         \centering \includegraphics[width=\linewidth]{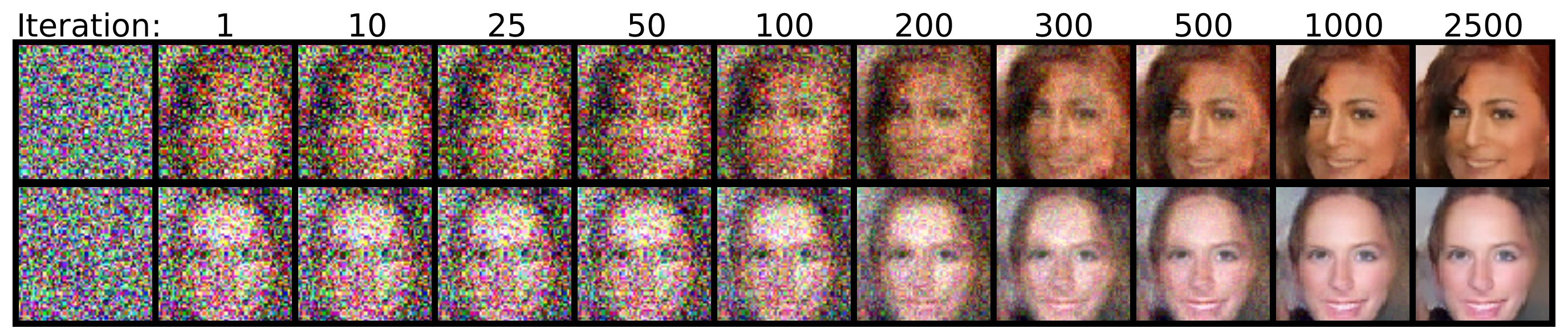}
     \end{subfigure}
     \caption{Gaussian noise (first column), Fashion-MNIST (top panel) and CelebA (bottom) samples at different iterations.}
     \label{fig:iteration}
     \vskip -0.1in
\end{figure}

\begin{figure}
     \centering
     \begin{subfigure}[b]{0.495\linewidth}
         \centering
         \includegraphics[width=\linewidth]{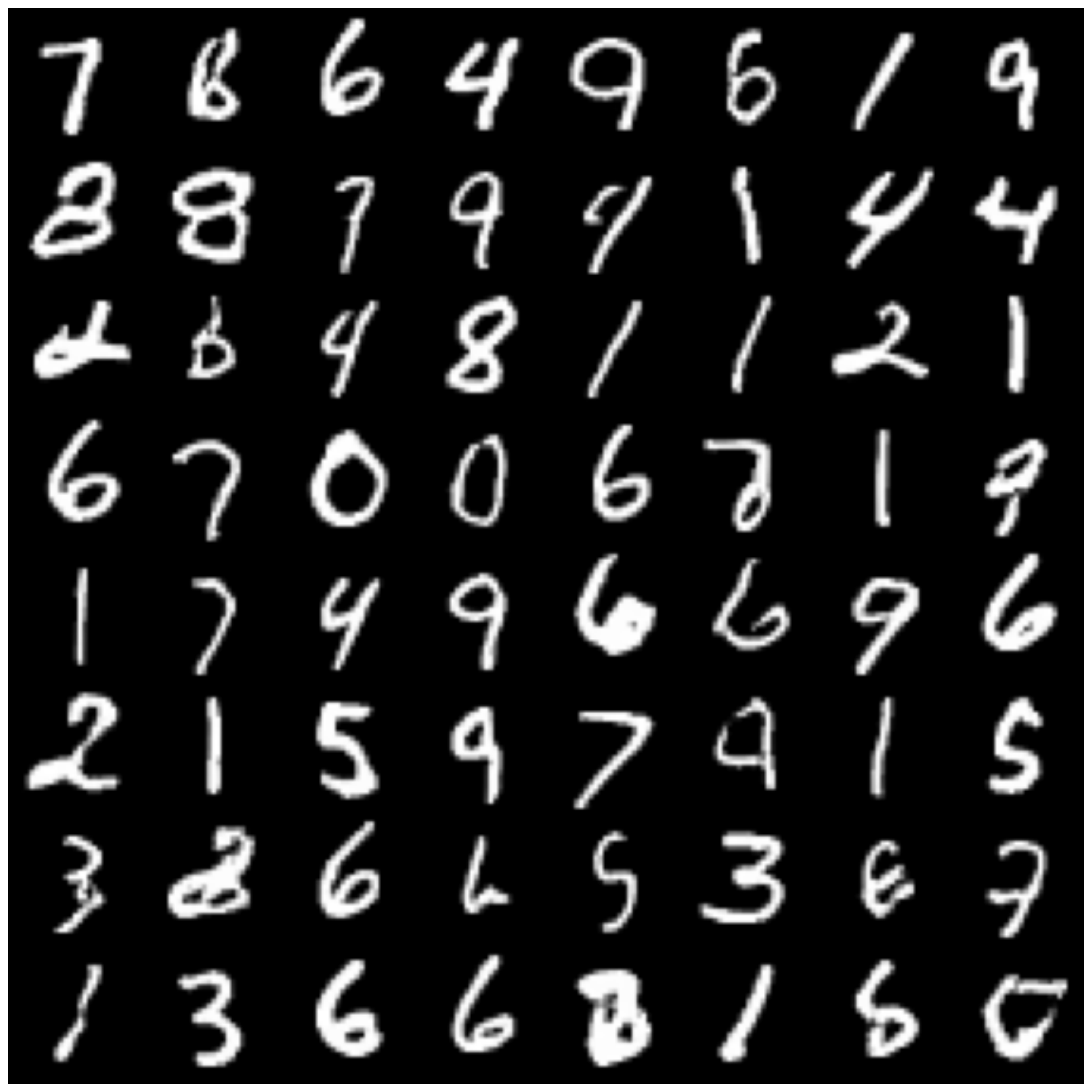}
         \caption{MNIST (T=0.85)}
     \end{subfigure}
     \hfill
     \begin{subfigure}[b]{0.495\linewidth}
         \centering
         \includegraphics[width=\linewidth]{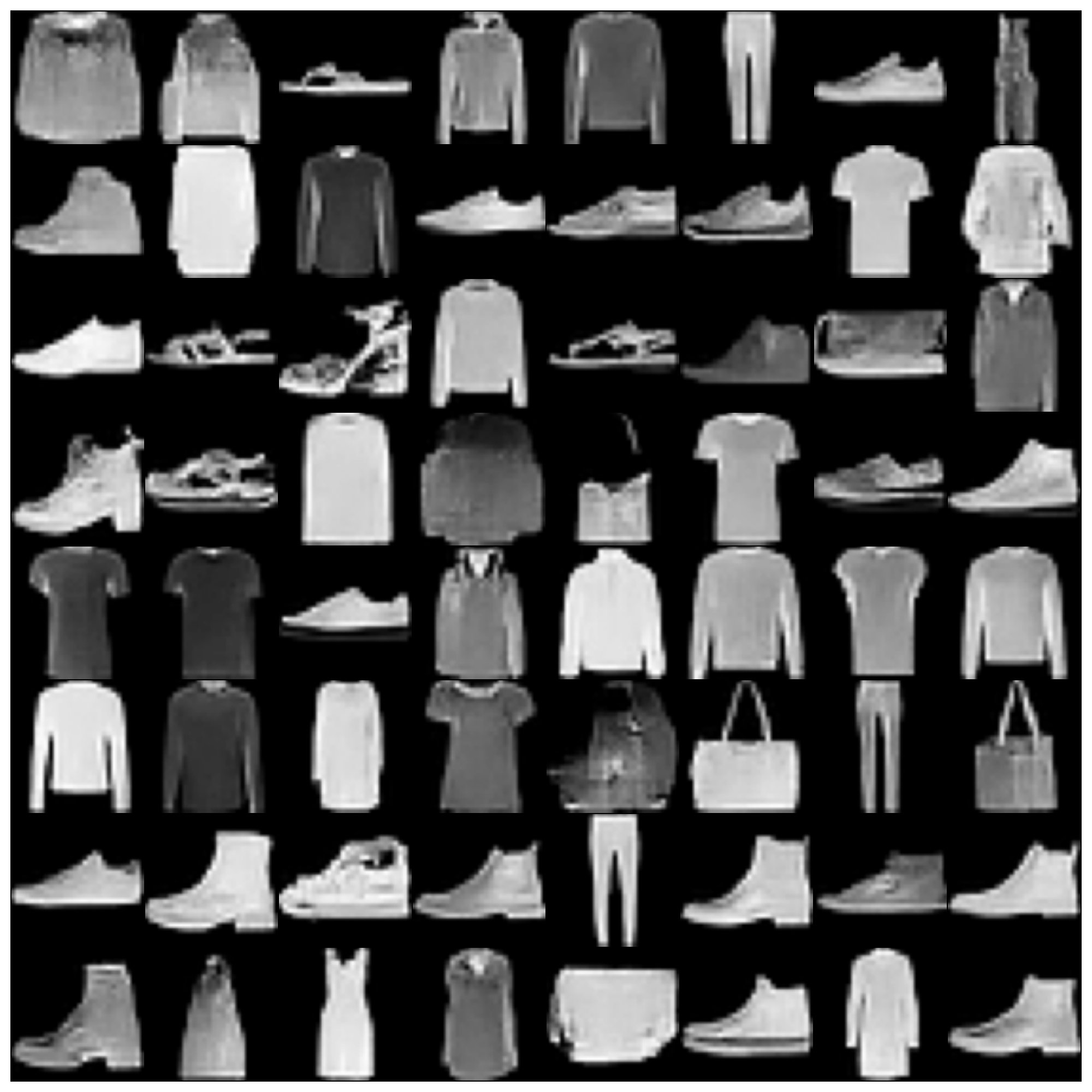}
         \caption{Fashion-MNIST (T=0.85)}
     \end{subfigure}
     \hfill
     \begin{subfigure}[b]{0.495\linewidth}
         \centering
         \includegraphics[width=\linewidth]{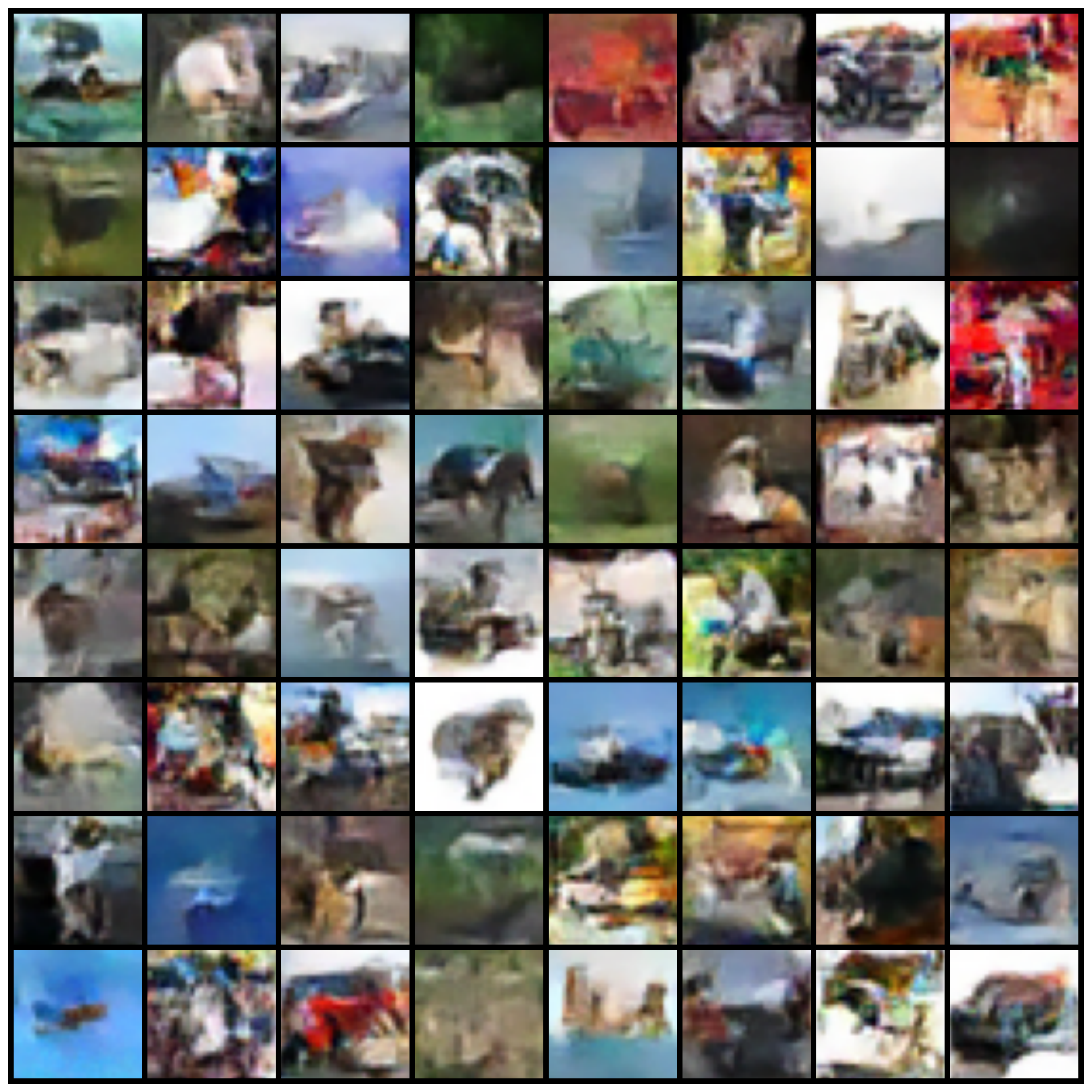}
         \caption{CIFAR-10 (T=1)}
     \end{subfigure}
     \hfill
     \begin{subfigure}[b]{0.495\linewidth}
         \centering
         \includegraphics[width=\linewidth]{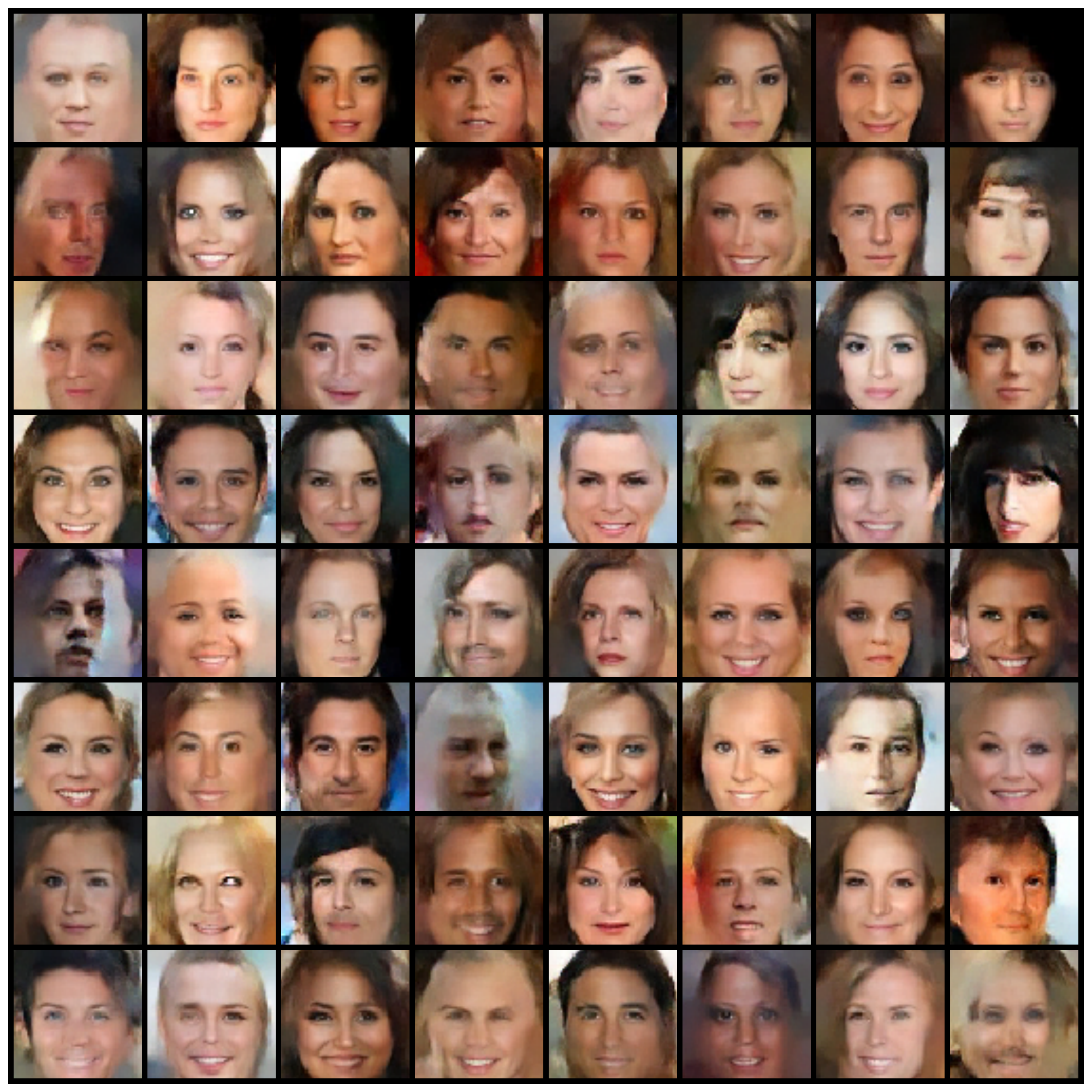}
         \caption{CelebA (T=0.85)}
     \end{subfigure}
     \caption{Random samples from SIG.}
     \label{fig:sample}
     \vskip -0.1in
\end{figure}

\begin{figure}
     \centering
     \begin{subfigure}[]{0.08\linewidth}
         \centering \includegraphics[width=\linewidth]{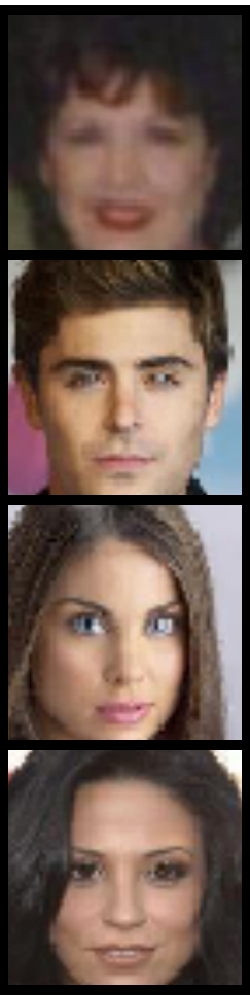}
     \end{subfigure}
     \hfill
     \begin{subfigure}[]{0.8\linewidth}
         \centering \includegraphics[width=\linewidth]{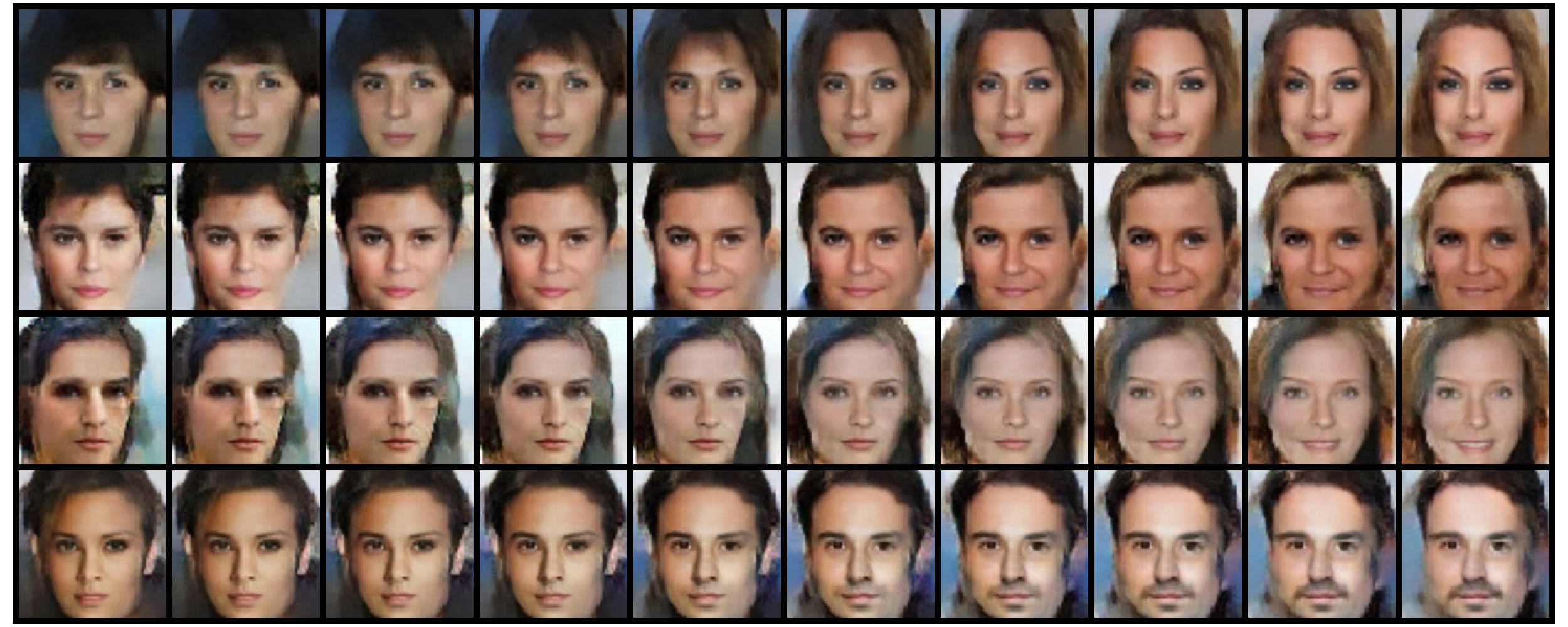}
     \end{subfigure}
     \hfill
     \begin{subfigure}[]{0.08\linewidth}
         \centering \includegraphics[width=\linewidth]{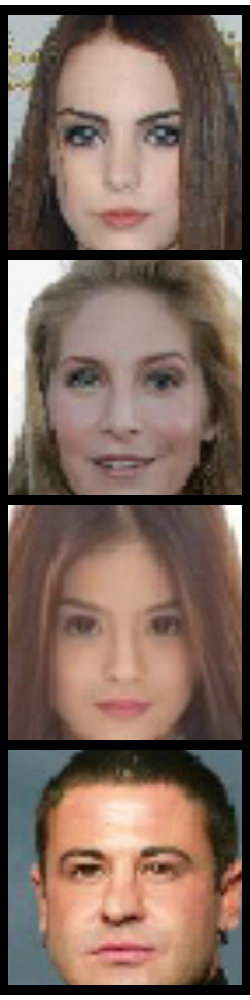}
     \end{subfigure}
     \caption{Middle: interpolations between CelebA samples from SIG. Left and right: the corresponding nearest training data.}
     \label{fig:interpolation}
     \vskip -0.1in
\end{figure}

\subsection{Improving the Samples of Other Generative Models}

\label{subsec:improve}


\begin{figure}
     \centering
     \begin{subfigure}[]{0.495\linewidth}
         \centering \includegraphics[width=\linewidth]{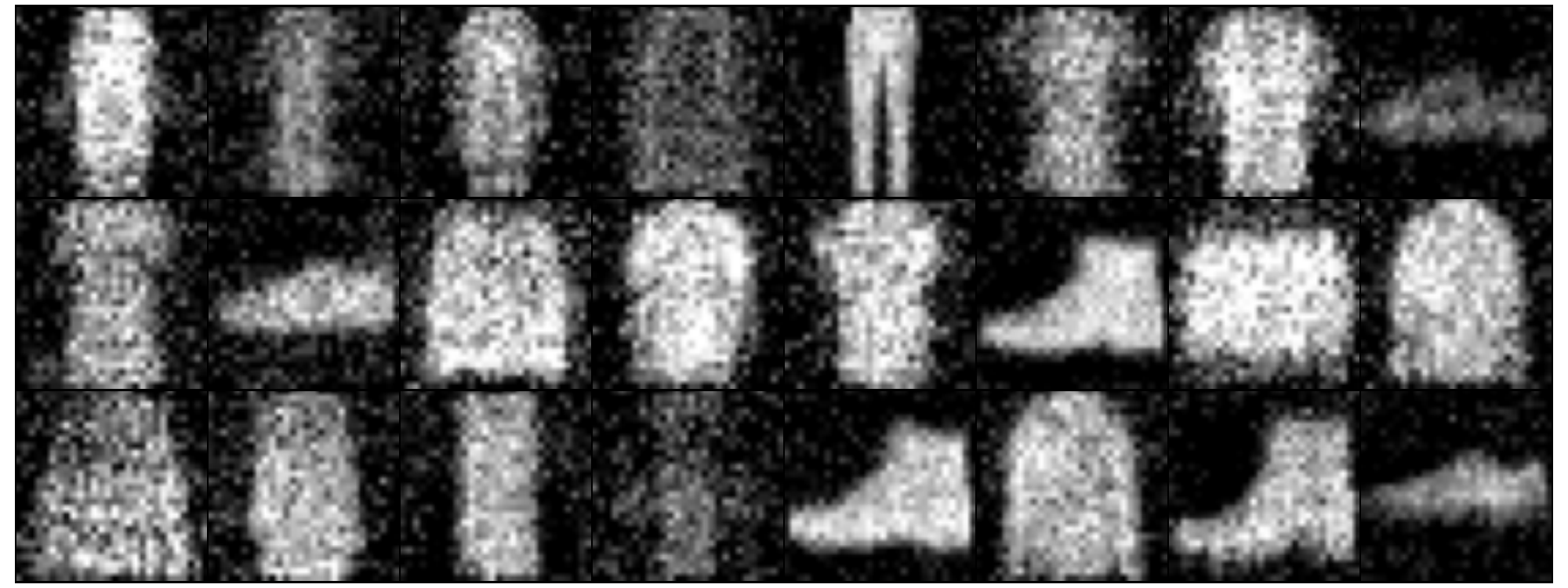}
     \end{subfigure}
     \hfill
     \begin{subfigure}[]{0.495\linewidth}
         \centering \includegraphics[width=\linewidth]{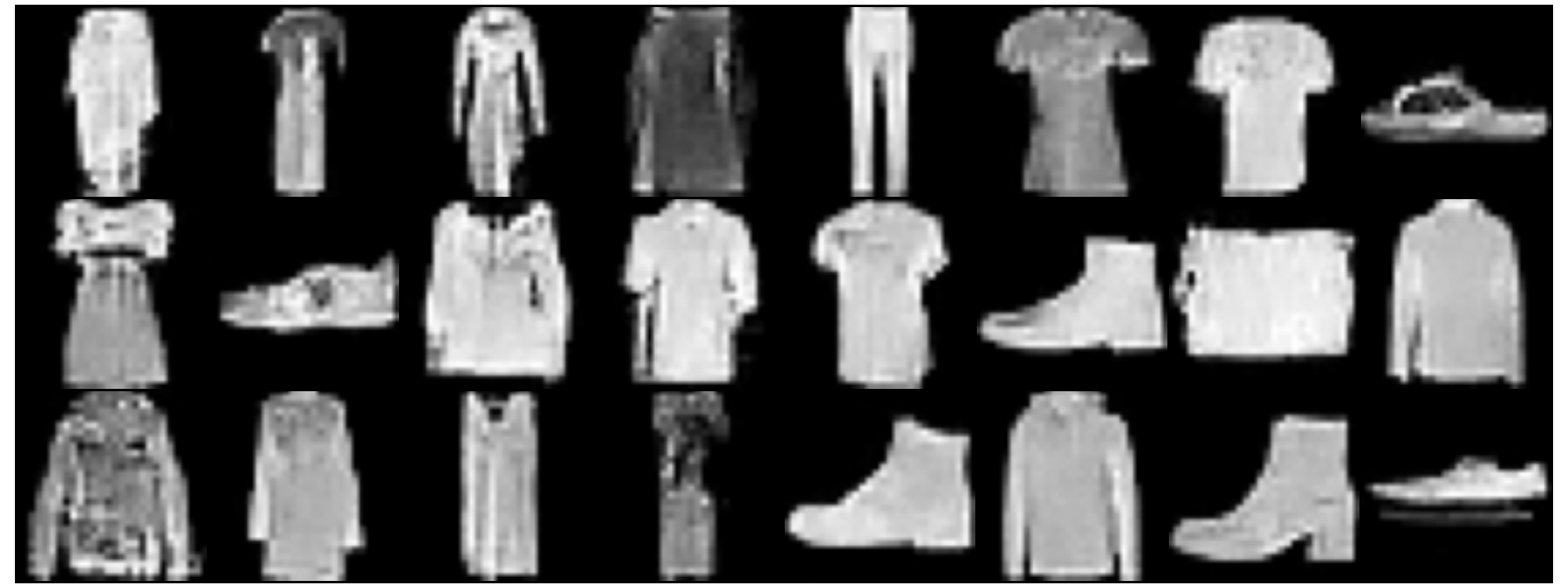}
     \end{subfigure}
     \hfill
     \begin{subfigure}[]{0.495\linewidth}
         \centering \includegraphics[width=\linewidth]{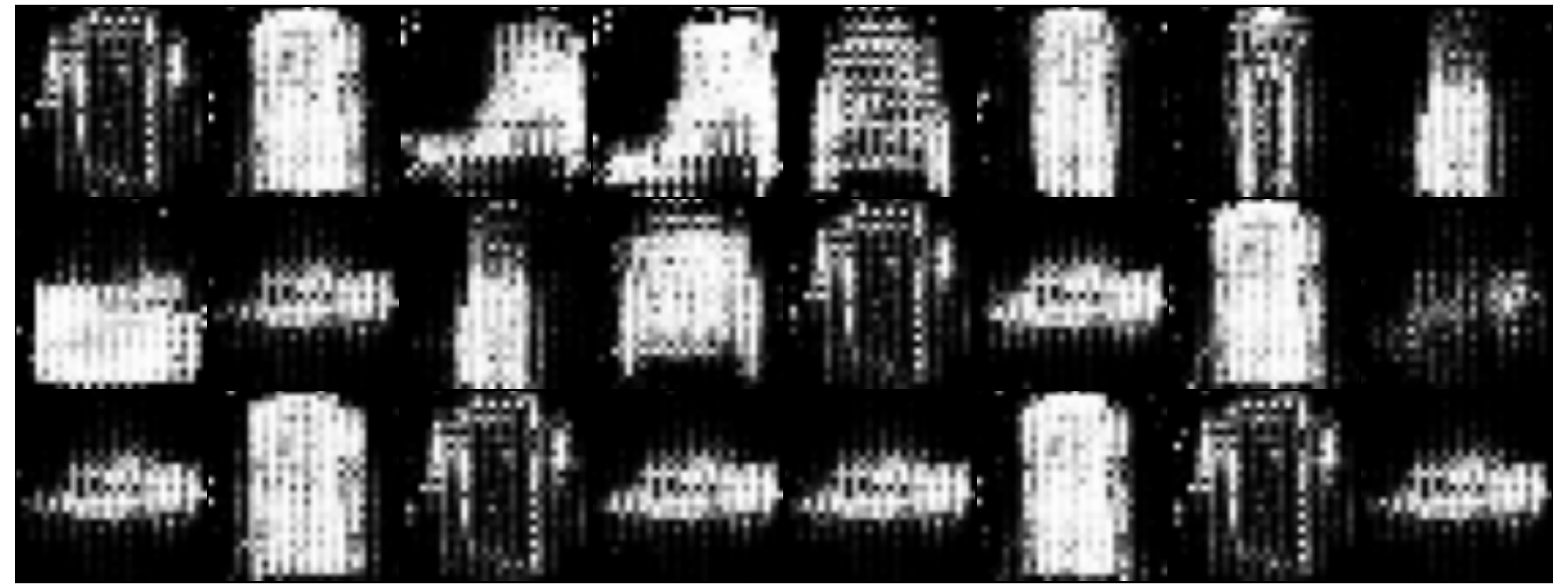}
     \end{subfigure}
     \hfill
     \begin{subfigure}[]{0.495\linewidth}
         \centering \includegraphics[width=\linewidth]{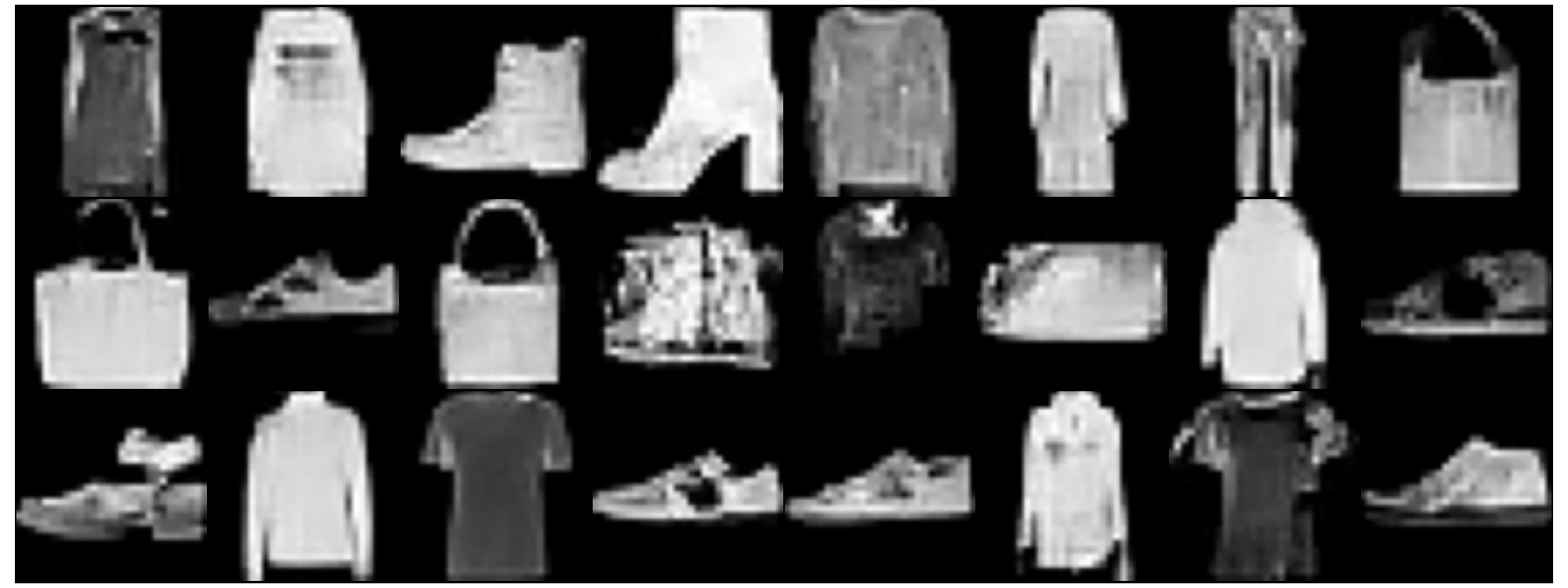}
     \end{subfigure}
     \hfill
     \begin{subfigure}[]{0.495\linewidth}
         \centering \includegraphics[width=\linewidth]{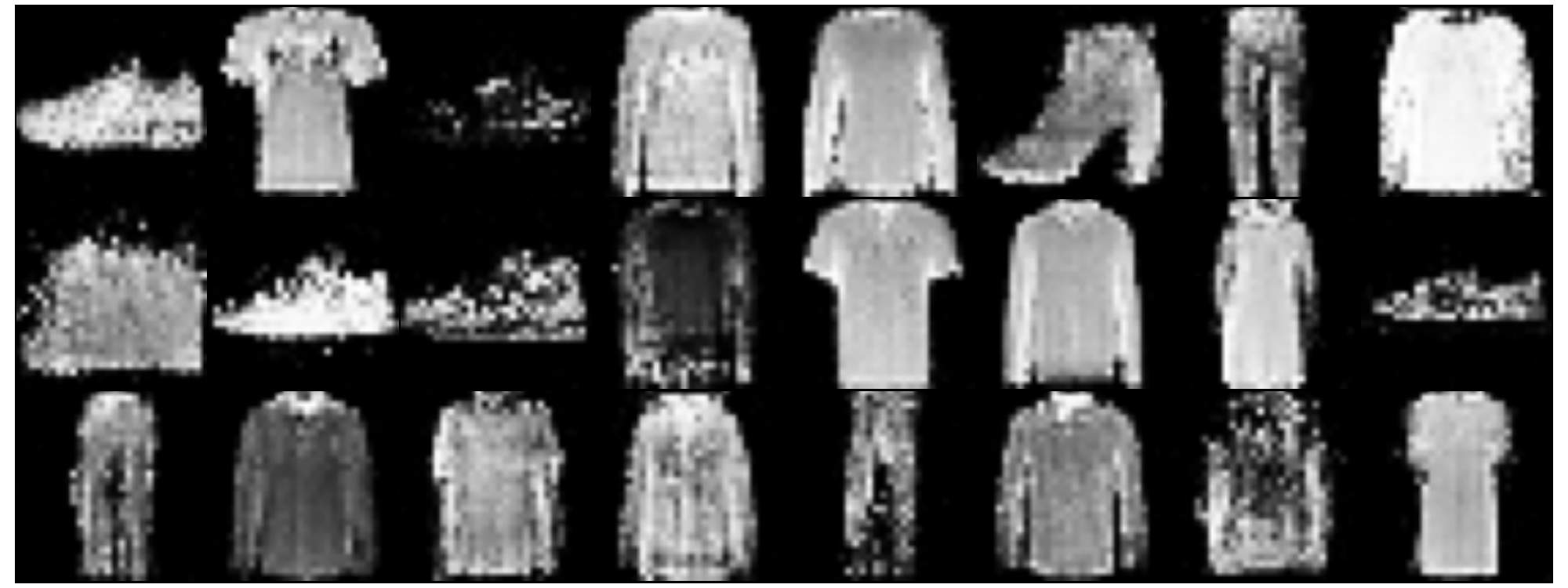}
     \end{subfigure}
     \hfill
     \begin{subfigure}[]{0.495\linewidth}
         \centering \includegraphics[width=\linewidth]{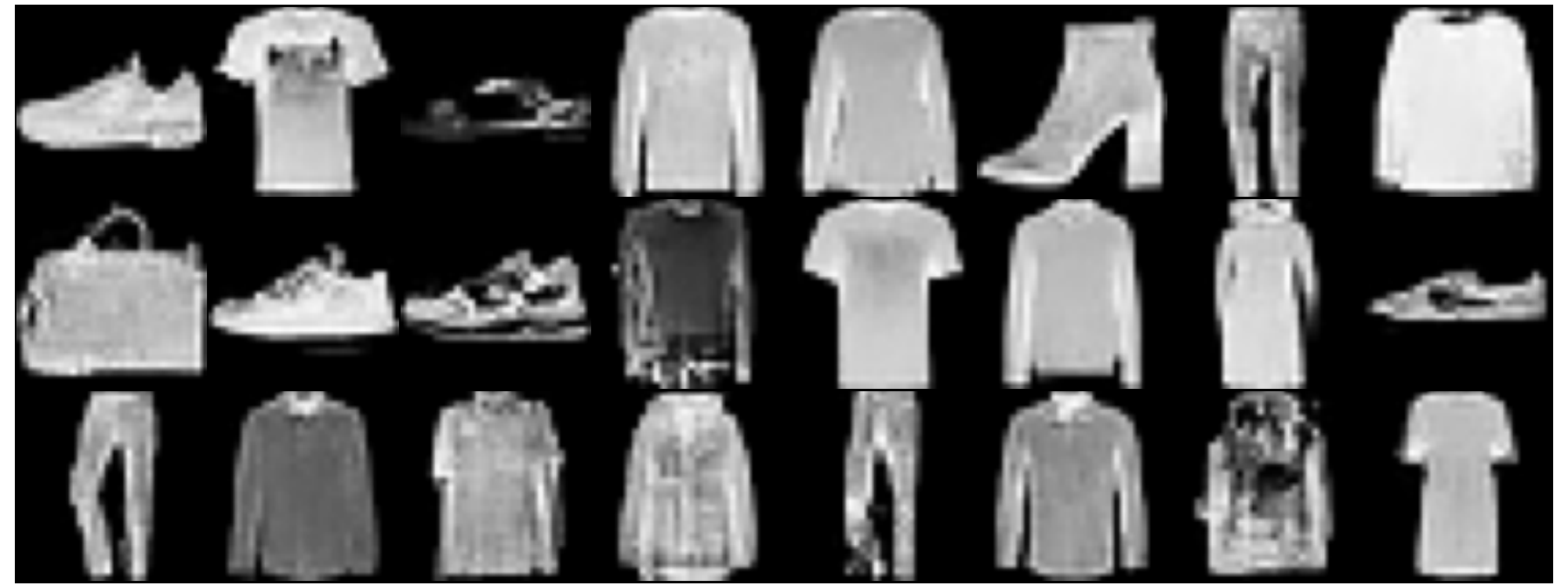}
     \end{subfigure}
    \caption{Fashion-MNIST samples before (left panel) and after SIG improvement (right panel). Top: SWF. Middle: Flow-GAN (ADV). Bottom: MAF.}
    \label{fig:improve}
    \vskip -0.1in
\end{figure}

Since SIG is able to transform any base distribution to the target distribution, it can also be used as a "Plug-and-Play" tool to improve the samples of other generative models. To demonstrate this, we train SWF, Flow-GAN(ADV) and MAF(5) on Fashion-MNIST with the default architectures in their papers, and then we apply 240 SIG iterations ($30\%$ of the total number of iterations in Section \ref{subsec:samples}) to improve the sample quality. In Figure \ref{fig:improve} we compare the samples before and after SIG improvement. Their FID scores improve from $207.6$, $216.9$ and $81.2$ to $23.9$, $21.2$ and $16.6$, respectively. These results can be further improved by adding more SIG iterations.

\subsection{Out of Distribution (OoD) Detection}

\label{subsec:ood}

\begin{table}[ht]
  \caption{OoD detection accuracy quantified by the AUROC of data $p(x)$ trained on Fashion-MNIST.  }
  \label{tab:auroc}
  \vskip 0.15in
  \centering
  \begin{tabular}{>{\centering}c|>{\centering}c>{\centering\arraybackslash}c}
  \toprule
    Method          & MNIST  & OMNIGLOT 
    \\
    \midrule\midrule
    SIG (this work)               &  \textbf{0.980} & \textbf{0.993}  
    \\
    GIS (this work) & 0.824 & 0.891 
    \\
    PixelCNN++ & 0.089 & -\\
    IWAE          &  0.423   & 0.568   
    \\
    \bottomrule
  \end{tabular}
\end{table}

OoD detection with generative models has recently attracted a lot of attention, since the $\log p$ estimates of NF and VAE have been shown to be poor OoD detectors:  different generative models can assign higher probabilities to OoD data than to In Distribution (InD) training data \citep{nalisnick2018deep}. One combination of datasets for which this has been observed is Fashion-MNIST and MNIST, where a model trained on the former assigns higher density to the latter. 


SINF does not train on the likelihood $p(x)$, which is an advantage for OoD. 
Likelihood is sensitive to the smallest variance directions \citep{ren2019likelihood}: for example, a zero variance pixel leads to an infinite $p(x)$, and noise must be added to regularize it. But zero variance directions contain little or no information on the global structure of the image. SINF objective is more sensitive to the meaningful global structures that can separate between OoD and InD. Because the patch based approach ignores the long range correlations and results in poor OoD, we use vanilla SINF without patch based approach. We train the models on F-MNIST, and then evaluate anomaly detection on test data of MNIST and OMNIGLOT \citep{lake2015human}. 
In Table~\ref{tab:auroc} we compare our results to maximum likelihood $p(x)$ models PixelCNN++\citep{salimans2017pixelcnn++, ren2019likelihood}, 
and IWAE \citep{choi2018waic}. Other models that perform well 
include VIB and WAIC \citep{choi2018waic}, which achieve 0.941, 0.943 and 0.766, 0.796, for MNIST and OMNIGLOT, respectively (below our SIG results). 
For the MNIST case \citet{ren2019likelihood} obtained 0.996 using the likelihood ratio between the model and its perturbed version, but they require fine-tuning on some additional OoD dataset, which may not be available in OoD applications. Lower dimensional latent space PAE \citep{bohm2020probabilistic} achieves 0.997 and 0.981 for MNIST and OMNIGLOT, respectively, while VAE based 
likelihood regret \citep{xiao2020likelihood} achieves
0.988 on MNIST, but requires additional (expensive)
processing. 

\section{Conclusions}

We introduce sliced iterative normalizig flow (SINF) that uses
Sliced Optimal Transport to iteratively transform data distribution to a Gaussian (GIS) or the other way around (SIG). To the best of our knowledge, SIG is the first greedy deep learning algorithm that is competitive with the SOTA generators in high dimensions, while GIS achieves comparable results on density estimation with current NF models, but is more stable, faster to train, and achieves higher $p(x)$ when trained on small training sets, even though it does not train on $p(x)$. It also achieves better OoD performance. SINF is very stable to train, has very few hyperparameters, and is very insensitive to their choice (see appendix).
SINF has deep neural network architecture, but its approach deviates significantly from the current DL paradigm, as it does not use concepts such as mini-batching, stochastic gradient descent and gradient back-propagation through deep layers. 
SINF is an existence proof that greedy DL without these ingredients can be state of the art for modern high dimensional ML applications. Such approaches thus deserve more detailed investigations that may have an impact on the theory and practice of DL.

\section*{Acknowledgements}

We thank He Jia for providing his code on Iterative Gaussianization, and for helpful discussions. We thank Vanessa Boehm and Jascha Sohl-Dickstein for comments on the manuscript. This material is based upon work supported by the National Science Foundation under Grant Numbers 1814370 and NSF 1839217, and by NASA under Grant Number 80NSSC18K1274. 

\bibliography{SIGGIS_reference}
\bibliographystyle{icml2021}


\newpage

\appendix

\twocolumn[
\icmltitle{Sliced Iterative Normalizing Flows \\
\small SUPPLEMENTARY DOCUMENT}

\icmlkeywords{Generative Models, Normalizing Flow, Optimal Transport, Sliced Wasserstein Distance}

]

\section{Proofs}
\begin{prop}
\label{prop:max-K-SWD}
Let $P_p(\Omega)$ be the set of Borel probability measures with finite p’th moment on metric space $(\Omega, d)$. The maximum K-sliced p-Wasserstein distance is a metric over $P_p(\Omega)$.
\end{prop}

\begin{proof}
We firstly prove the triangle inequality.
Let $\mu_1$, $\mu_2$ and $\mu_3$ be probability measures in $P_p(\Omega)$ with probability density function $p_1$, $p_2$ and $p_3$, respectively. 
Let $\{\theta_1^*, \cdots, \theta_K^*\} = \argmax_{\{\theta_1, \cdots, \theta_K\}\  \mathrm{orthonormal}}\\ \left(\frac{1}{K} \sum_{k=1}^K W_p^p((\mathcal{R}p_1)(\cdot,\theta_k), (\mathcal{R}p_3)(\cdot,\theta_k))\right)^{\frac{1}{p}}$; then
\begin{equation}
\begin{aligned}
    &\max {\textrm -} K {\textrm -} SW_p(p_1, p_3)\\
    =& \max_{\{\theta_1, \cdots, \theta_K\}\  \mathrm{orthonormal}} \\ &\left(\frac{1}{K} \sum_{k=1}^K W_p^p((\mathcal{R}p_1)(\cdot,\theta_k), (\mathcal{R}p_3)(\cdot,\theta_k))\right)^{\frac{1}{p}} \\
    =& \left(\frac{1}{K} \sum_{k=1}^K W_p^p((\mathcal{R}p_1)(\cdot,\theta_k^*), (\mathcal{R}p_3)(\cdot,\theta_k^*))\right)^{\frac{1}{p}} \\
    \leq & \left(\frac{1}{K} \sum_{k=1}^K [W_p((\mathcal{R}p_1)(\cdot,\theta_k^*), (\mathcal{R}p_2)(\cdot,\theta_k^*)) \right.\\
    & \left. + W_p((\mathcal{R}p_2)(\cdot,\theta_k^*), (\mathcal{R}p_3)(\cdot,\theta_k^*)) \vphantom{\sum_{k=1}^K}]^p\right) ^{\frac{1}{p}}\\
    \leq & \left(\frac{1}{K} \sum_{k=1}^K W_p^p((\mathcal{R}p_1)(\cdot,\theta_k^*), (\mathcal{R}p_2)(\cdot,\theta_k^*)) \right)^{\frac{1}{p}}\\
    & + \left(\frac{1}{K} \sum_{k=1}^K W_p^p((\mathcal{R}p_2)(\cdot,\theta_k^*), (\mathcal{R}p_3)(\cdot,\theta_k^*)) \right)^{\frac{1}{p}}\\
    \leq & \max_{\{\theta_1, \cdots, \theta_K\}\  \mathrm{orthonormal}}\\
    & \left(\frac{1}{K} \sum_{k=1}^K W_p^p((\mathcal{R}p_1)(\cdot,\theta_k), (\mathcal{R}p_2)(\cdot,\theta_k)) \right)^{\frac{1}{p}}\\
    & + \max_{\{\theta_1, \cdots, \theta_K\}\  \mathrm{orthonormal}}\\
    & \left(\frac{1}{K} \sum_{k=1}^K W_p^p((\mathcal{R}p_2)(\cdot,\theta_k), (\mathcal{R}p_3)(\cdot,\theta_k)) \right)^{\frac{1}{p}}\\
    = & \max {\textrm -} K {\textrm -} SW_p(p_1, p_2) + \max {\textrm -} K {\textrm -} SW_p(p_2, p_3) ,
\end{aligned}
\end{equation}
where the first inequality comes from the triangle inequality of Wasserstein distance, and the second inequality follows Minkowski inequality. Therefore $\max {\textrm -} K {\textrm -} SW_p$ satisfies the triangle inequality.

Now we prove the identity of indiscernibles. For any probability measures $\mu_1$ and $\mu_2$ in $P_p(\Omega)$ with probability density function $p_1$ and $p_2$, let \\
$\hat{\theta} = \argmax_{\theta\in \mathbb{S}^{d-1}} W_p((\mathcal{R}p_1)(\cdot,\theta), (\mathcal{R}p_2)(\cdot,\theta))$, and\\
$\{\theta_1^*, \cdots, \theta_K^*\} = \argmax_{\{\theta_1, \cdots, \theta_K\}\  \mathrm{orthonormal}}\\ \left(\frac{1}{K} \sum_{k=1}^K W_p^p((\mathcal{R}p_1)(\cdot,\theta_k), (\mathcal{R}p_2)(\cdot,\theta_k))\right)^{\frac{1}{p}}$, we have
\begin{equation}
\begin{aligned}
\label{eq:max-K-SWD-bound1}
    &\max {\textrm -} K {\textrm -} SW_p(p_1, p_2)\\
    =& \left(\frac{1}{K} \sum_{k=1}^K W_p^p((\mathcal{R}p_1)(\cdot,\theta_k^*), (\mathcal{R}p_2)(\cdot,\theta_k^*))\right)^{\frac{1}{p}} \\
    \leq & \left(\frac{1}{K} \sum_{k=1}^K W_p^p((\mathcal{R}p_1)(\cdot,\hat{\theta}), (\mathcal{R}p_2)(\cdot,\hat{\theta}))\right)^{\frac{1}{p}} \\
    = & W_p((\mathcal{R}p_1)(\cdot,\hat{\theta}), (\mathcal{R}p_2)(\cdot,\hat{\theta}))\\
    = & \max {\textrm -} SW_p(p_1, p_2) .
\end{aligned}
\end{equation}
On the other hand, let $\{\hat{\theta}, \tilde{\theta}_2, \cdots, \tilde{\theta}_K\}$ be a set of orthonormal vectors in $\mathbb{S}^{d-1}$ where the first element is $\hat{\theta}$, we have 
\begin{equation}
\begin{aligned}
\label{eq:max-K-SWD-bound2}
    &\max {\textrm -} K {\textrm -} SW_p(p_1, p_2)\\
    =& \left(\frac{1}{K} \sum_{k=1}^K W_p^p((\mathcal{R}p_1)(\cdot,\theta_k^*), (\mathcal{R}p_2)(\cdot,\theta_k^*))\right)^{\frac{1}{p}} \\
    \geq & \left(\frac{1}{K} \vphantom{\sum_{k=1}^K} W_p^p((\mathcal{R}p_1)(\cdot,\hat{\theta}), (\mathcal{R}p_2)(\cdot,\hat{\theta}))\right. \\
    & + \left. \frac{1}{K} \sum_{k=2}^K W_p^p((\mathcal{R}p_1)(\cdot,\tilde{\theta}_k), (\mathcal{R}p_2)(\cdot,\tilde{\theta}_k))\right)^{\frac{1}{p}} \\
    \geq & \left(\frac{1}{K}  W_p^p((\mathcal{R}p_1)(\cdot,\hat{\theta}), (\mathcal{R}p_2)(\cdot,\hat{\theta}))\right)^{\frac{1}{p}} \\
    = & (\frac{1}{K})^{\frac{1}{p}}  \max {\textrm -} SW_p(p_1, p_2) .
\end{aligned}
\end{equation}
Therefore we have $(\frac{1}{K})^{\frac{1}{p}}  \max {\textrm -} SW_p(p_1, p_2) \leq \\ \max {\textrm -} K {\textrm -} SW_p(p_1, p_2) \leq \max {\textrm -} SW_p(p_1, p_2)$. Thus $\max {\textrm -} K {\textrm -} SW_p(p_1, p_2)=0 \Leftrightarrow \max {\textrm -} SW_p(p_1, p_2)=0 \Leftrightarrow \mu_1 = \mu_2$, where we use the non-negativity and identity of indiscernibles of $\max {\textrm -} SW_p$.

Finally, the symmetry of $\max {\textrm -} K {\textrm -} SW_p$ can be proven using the fact that p-Wasserstein distance is symmetric:
\begin{equation}
\begin{aligned}
    &\max {\textrm -} K {\textrm -} SW_p(p_1, p_2)\\
    =& \left(\frac{1}{K} \sum_{k=1}^K W_p^p((\mathcal{R}p_1)(\cdot,\theta_k^*), (\mathcal{R}p_2)(\cdot,\theta_k^*))\right)^{\frac{1}{p}} \\
    =& \left(\frac{1}{K} \sum_{k=1}^K W_p^p((\mathcal{R}p_2)(\cdot,\theta_k^*), (\mathcal{R}p_1)(\cdot,\theta_k^*))\right)^{\frac{1}{p}} \\
    =& \max {\textrm -} K {\textrm -} SW_p(p_2, p_1).
\end{aligned}
\end{equation}
\end{proof}

\begin{proof}[Proof of Equation \ref{eq:jacobian}]
Let $\{\theta_1, \cdots, \theta_K, \cdots, \theta_d\}$ be a set of orthonormal basis in $\mathcal{R}^d$ where the first $K$ vectors are $\theta_1, \cdots, \theta_K$, respectively. Let $R_l=[\theta_1, \cdots, \theta_d]$ be an orthogonal matrix whose i-th column vector is $\theta_i$, $U_l=[\theta_{K+1}, \cdots, \theta_d]$. Since $A_l=[\theta_1, \cdots, \theta_K]$, we have $R_l=[A_l,U_l]$ (the concatenation of columns of $A$ and $U$). Let $\mathbf{I}^{d-K}=[\mathrm{id}_1, \cdots, \mathrm{id}_{d-K}]^T$ be a marginal transformation that consists of $d-K$ 1D identity transformation,  $\hat{\mathbf{\Psi}}_l=
\begin{bmatrix}
\mathbf{\Psi}_{l}\\
\mathbf{I}^{d-K}
\end{bmatrix}$, we have
\begin{equation}
\begin{aligned}
    X_{l+1} =& A_l \mathbf{\Psi}_{l}(A_l^TX_l) + X_l - A_lA_l^TX_l\\
    =& A_l \mathbf{\Psi}_{l}(A_l^TX_l) + R_lR_l^TX_l - A_lA_l^TX_l\\
    =& A_l \mathbf{\Psi}_{l}(A_l^TX_l) + [A_l,U_l] 
    \begin{bmatrix}
    A_l^T\\
    U_l^T
    \end{bmatrix}X_l - A_lA_l^TX_l\\
    =& A_l \mathbf{\Psi}_{l}(A_l^TX_l) + U_lU_l^TX_l\\
    =& A_l \mathbf{\Psi}_{l}(A_l^TX_l) + U_l \mathbf{I}^{d-K}(U_l^TX_l)\\
    =& [A_l, U_l]
    \begin{bmatrix}
    \mathbf{\Psi}_{l}\\
    I^{d-K}
    \end{bmatrix}
    \left([A_l, U_l]^TX_l\right)\\
    =& R_l\hat{\mathbf{\Psi}}_l(R_l^TX_l).
\end{aligned}
\end{equation}

Since $R_l$ is an orthogonal matrix with determinant $\pm1$, and the Jacobian of the marginal transformation $\hat{\mathbf{\Psi}}_l$ is diagonal, the Jacobian determinant of the above equation can be written as
\begin{equation}
\begin{aligned}
 \det(\frac{\partial X_{l+1}}{\partial X_l}) =& \prod_{k=1}^K \frac{d\Psi_{lk}(x)}{dx} \cdot \prod_{k=1}^{d-K} \frac{d(\mathrm{id}_k(x))}{dx}\\
 =& \prod_{k=1}^K \frac{d\Psi_{lk}(x)}{dx} .
\end{aligned}
\end{equation}
\end{proof}

\section{Monotonic Rational Quadratic Spline}
\label{sec:RQspline}

Monotonic Rational Quadratic Splines \citep{gregory1982piecewise, durkan2019neural} approximate the function in each bin with the quotient of two quadratic polynomials. They are monotonic, contineously differentiable, and can be inverted analytically. The splines are parametrized by the coordinates and derivatives of $M$ knots: $\{(x_m, y_m, y_m')\}_{m=1}^{M}$, with $x_{m+1}>x_m$, $y_{m+1}>y_m$ and $y_{m}'>0$. Given these parameters, the function in bin $m$ can be written as \citep{durkan2019neural}
\begin{equation}
    y = y_m  + (y_{m+1}-y_m)\frac{s_m\xi^2+y_m'\xi(1-\xi)}{s_m+\sigma_m\xi(1-\xi)},
\end{equation}
where $s_m=(y_{m+1}-y_m)/(x_{m+1}-x_m)$, $\sigma_m=y_{m+1}'+y_m'-2s_m$ and $\xi=(x-x_m)/(x_{m+1}-x_m)$. The derivative is given by
\begin{equation}
    \frac{dy}{dx} =  \frac{s_m^2[y_{m+1}'\xi^2+2s_m\xi(1-\xi)+y_m'(1-\xi)^2]}{[s_m+\sigma_m\xi(1-\xi)]^2}.
\end{equation}
Finally, the inverse can be calculated with
\begin{equation}
    x = x_m + (x_{m+1}-x_m)\frac{2c}{-b-\sqrt{b^2-4ac}},
\end{equation}
where $a = (s_m-y_m') + \zeta\sigma_m$, $b = y_m' - \zeta\sigma_m$, $c = -s_m\zeta$ and $\zeta = (y-y_m)/(y_{m+1}-y_m)$. The derivation of these formula can be found in Appendix A of \citet{durkan2019neural}.

In our algorithm the coordinates of the knots are determined by the quantiles of the marginalized PDF (see Algorithm \ref{alg:NF}). The derivative $y_m'$ $(1<m<M)$ is determined by fitting a local quadratic polynomial to the neighboring knots $(x_{m-1}, y_{m-1})$, $(x_m, y_m)$, and $(x_{m+1}, y_{m+1})$:
\begin{equation}
    y_m' = \frac{s_{m-1}(x_{m+1}-x_m)+s_{m}(x_{m}-x_{m-1})}{x_{m+1}-x_{m-1}}.
\end{equation}
The function outside $[x_1, x_M]$ is linearly extrapolated with slopes $y_1'$ and $y_M'$. In SIG, $y_1'$ and $y_M'$ are fixed to 1, while in GIS they are fitted to the samples that fall outside $[x_1, x_M]$.

We use $M=400$ knots in SIG to interpolate each $\Psi_{l,k}$, while in GIS we allow $M$ to vary between $[50,200]$, depending on the dataset size $M=\sqrt{N_{\mathrm{train}}}$. The performance is insensitive to these choices, as long as $M$ is large enough to fully characterize the 1D transformation $\Psi_{l,k}$.

\section{Optimization on the Stiefel Manifold}
\label{sec:stiefel}

The calculation of max K-SWD (Equation \ref{eq:maxKSWp}) requires optimization under the constraints that $\{\theta_1, \cdots, \theta_K\}$ are orthonormal vectors, or equivalently, $A^TA=I_K$ where $A = [\theta_1, \cdots, \theta_K]$ is the matrix whose i-th column vector is $\theta_i$. As suggested by \citet{tagare2011notes}, the optimization of matrix $A$ can be performed by doing gradient ascent on the Stiefel Manifold:
\begin{equation}
    \label{eq:Cayley1}
    A_{(j+1)} = \left( I_d + \frac{\tau}{2}B_{(j)}\right)^{-1}\left( I_d - \frac{\tau}{2}B_{(j)}\right)A_{(j)} ,
\end{equation}
where $A_{(j)}$ is the weight matrix at gradient descent iteration $j$ (which is different from the iteration $l$ of the algorithm), $\tau$ is the learning rate, which is determined by backtracking line search, $B=GA^T-AG^T$, and $G$ is the negative gradient matrix $G=[-\frac{\partial \mathcal{F}}{\partial A_{p,q}}] \in \mathbb{R}^{d\times K}$. Equation \ref{eq:Cayley1} has the properties that $A_{(j+1)} \in V_K(\mathbb{R}^d)$, and that the tangent vector $\frac{dA_{(j+1)}}{d\tau}|_{\tau=0}$ is the projection of gradient $[\frac{\partial \mathcal{F}}{\partial A_{p,q}}]$ onto $T_{A_{(j)}}(V_K(\mathbb{R}^d))$ (the tangent space of $V_K(\mathbb{R}^d)$ at $A_{(j)}$) under the canonical inner product \citep{tagare2011notes}. 

However, Equation \ref{eq:Cayley1} requires the inverse of a $d\times d$ matrix, which is computationally expensive in high dimensions. The matrix inverse can be simplified using the Sherman-Morrison-Woodbury formula, which results in the following equation \citep{tagare2011notes}:
\begin{equation}
    \label{eq:Cayley2}
    A_{(j+1)} = A_{(j)} - \tau U_{(j)} (I_{2K}+\frac{\tau}{2}V_{(j)}^TU_{(j)})^{-1}V_{(j)}^TA_{(j)} ,
\end{equation}
where $U=[G, A]$ (the concatenation of columns of $G$ and $A$) and $V=[A, -G]$. Equation \ref{eq:Cayley2} only involves the inverse of a $2K\times2K$ matrix. For high dimensional data (e.g. images), we use a relatively small $K$ to avoid the inverse of large matrices. A large $K$ leads to faster training, but one would converge to similar results with a small $K$ using more iterations. In Appendix \ref{sec:ablation} we show that the convergence is insensitive to the choice of $K$.

\section{Hyperparameter study and ablation analysis}
\label{sec:ablation}
Here we study the sensitivity of SINF to hyperparameters and perform 
ablation analyses. 
\subsection{Hyperparameter $K$, objective function, and patch based approach}


\begin{figure}[htb]
     \centering
     \begin{subfigure}[b]{\linewidth}
         \centering
         \includegraphics[width=\linewidth]{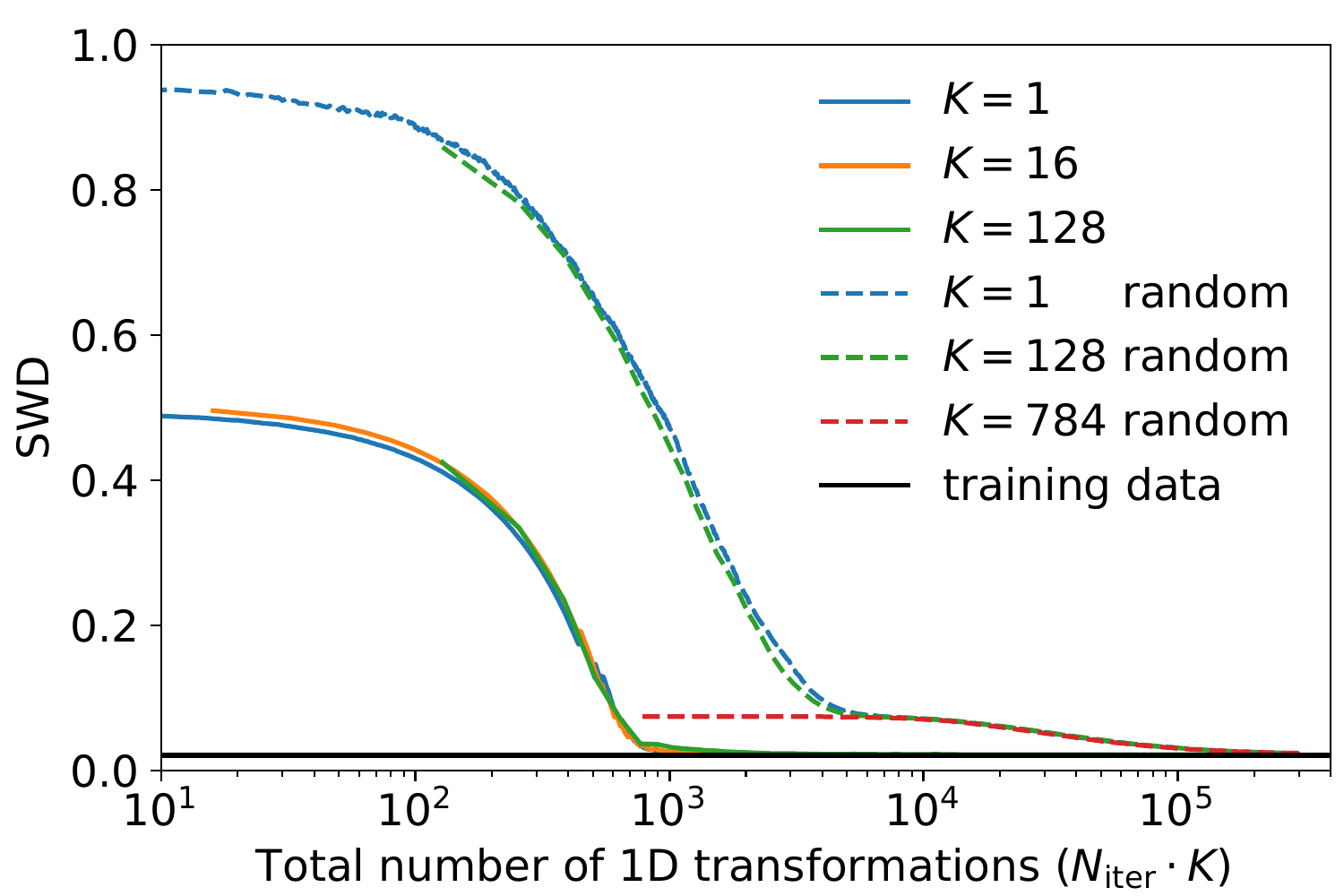}
     \end{subfigure}
     \hfill
     \begin{subfigure}[b]{\linewidth}
         \centering
         \includegraphics[width=\linewidth]{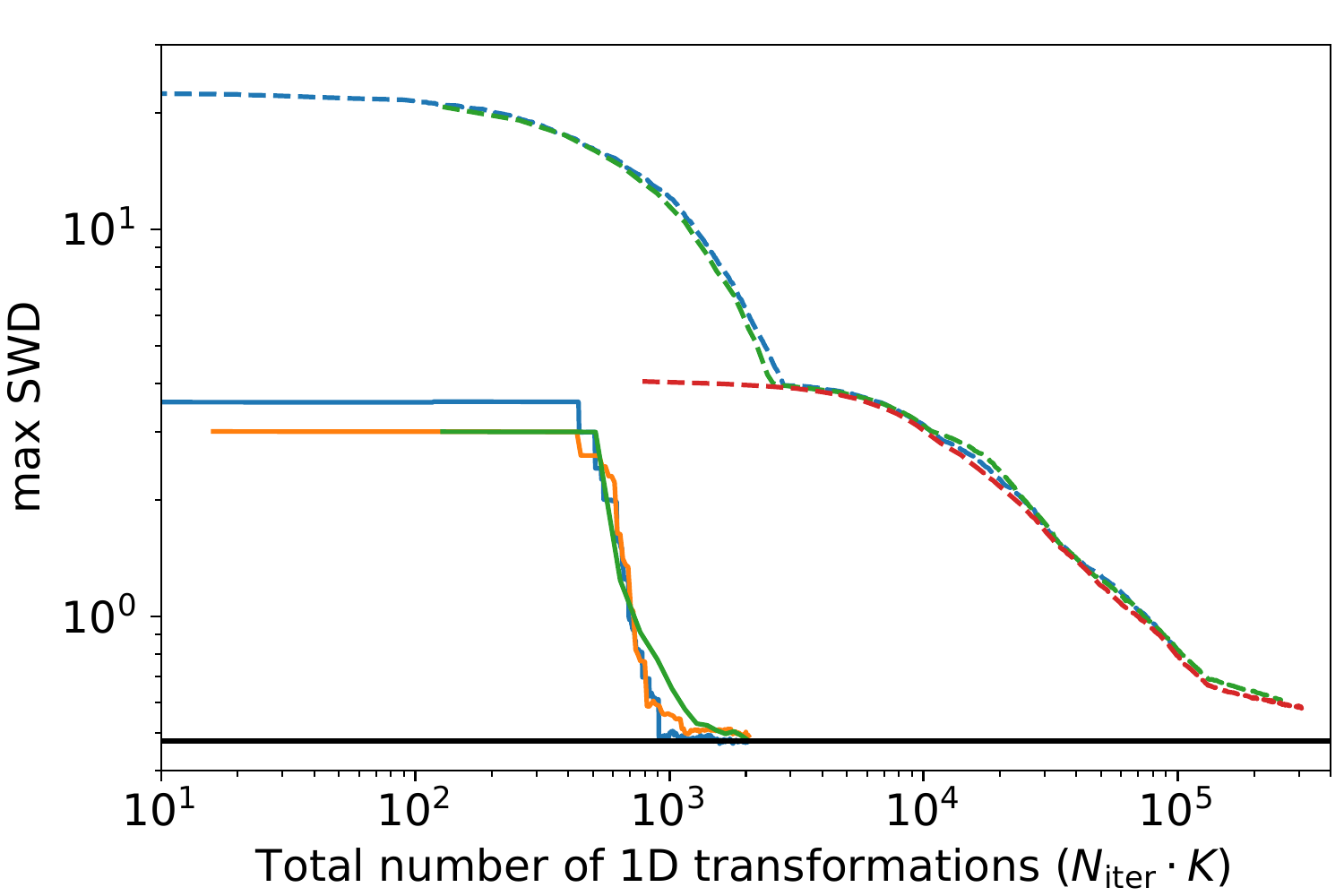}
     \end{subfigure}
     \caption{Sliced Wasserstein Distance (SWD, top panel) and Max-Sliced Wasserstein Distance (max SWD, bottom panel) between the MNIST test data and model samples as a function of total number of marginal transformations. The legend in the top panel also applies to the bottom panel. The SWD and max SWD between the training data and test data is shown in the horizontal solid black lines. The lines with "random" indicate that the axes are randomly chosen (like RBIG) instead of using the axes of max K-SWD. We also test $K=2,\ 4,\ 8,\ 32,$ and $64$. Their curves overlap with $K=1,\ 16$ and $128$ and are not shown in the plot.}
     \label{fig:convergence}
     \vskip -0.15in
\end{figure}

\begin{figure}[htb]
     \centering
     \begin{subfigure}[b]{\linewidth}
         \centering
         \includegraphics[width=\linewidth]{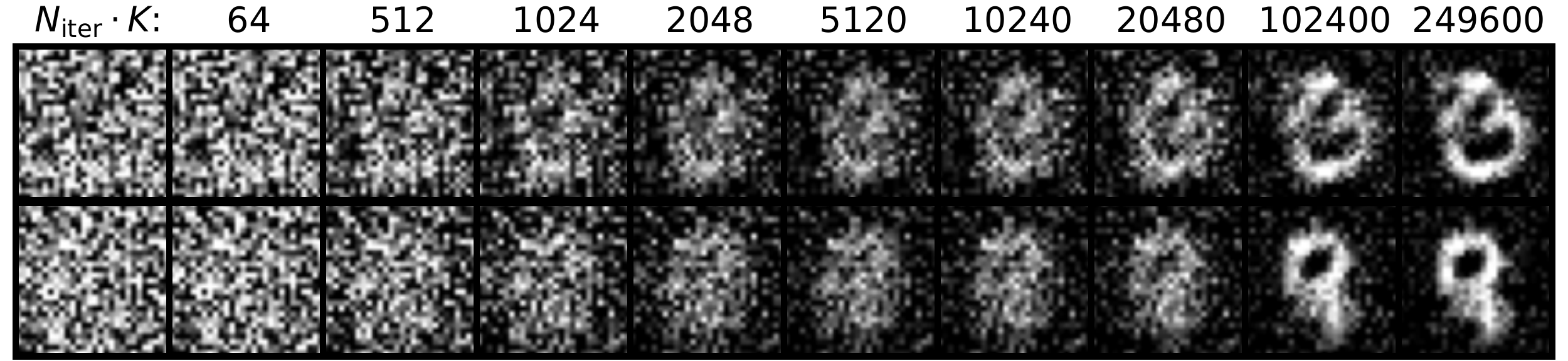}
     \end{subfigure}
     \hfill
     \begin{subfigure}[b]{\linewidth}
         \centering
         \includegraphics[width=\linewidth]{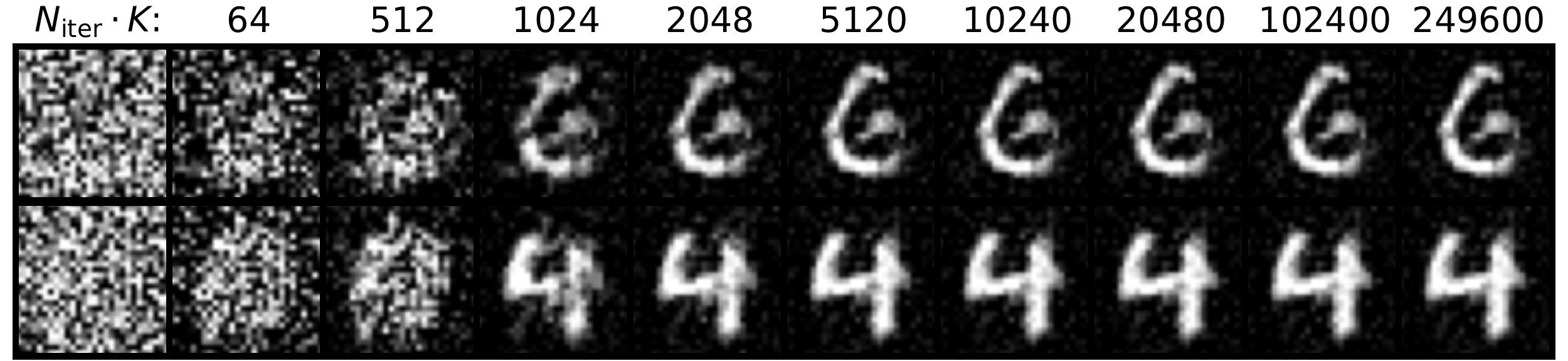}
     \end{subfigure}
     \begin{subfigure}[b]{\linewidth}
         \centering
         \includegraphics[width=\linewidth]{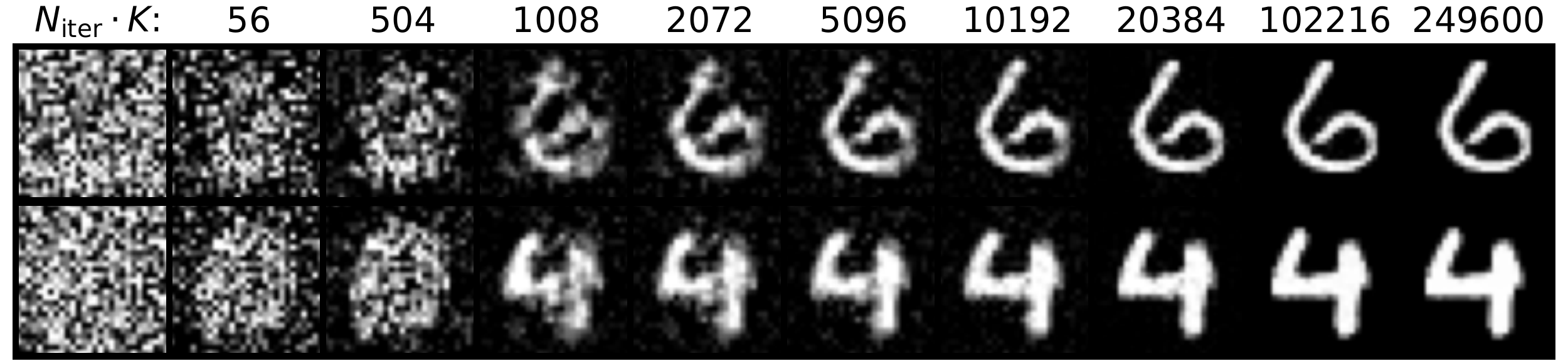}
     \end{subfigure}
    \caption{Top panel: SIG samples with random axes ($K=64$). Middle panel: SIG samples with optimized axes ($K=64$). Bottom panel: SIG samples with optimized axes and patch based hierarchical approach. The numbers above each panel indicate the number of marginal transformations.}
    \label{fig:ablation}
\end{figure}

\begin{figure}[htb]
     \centering
     \includegraphics[width=\linewidth]{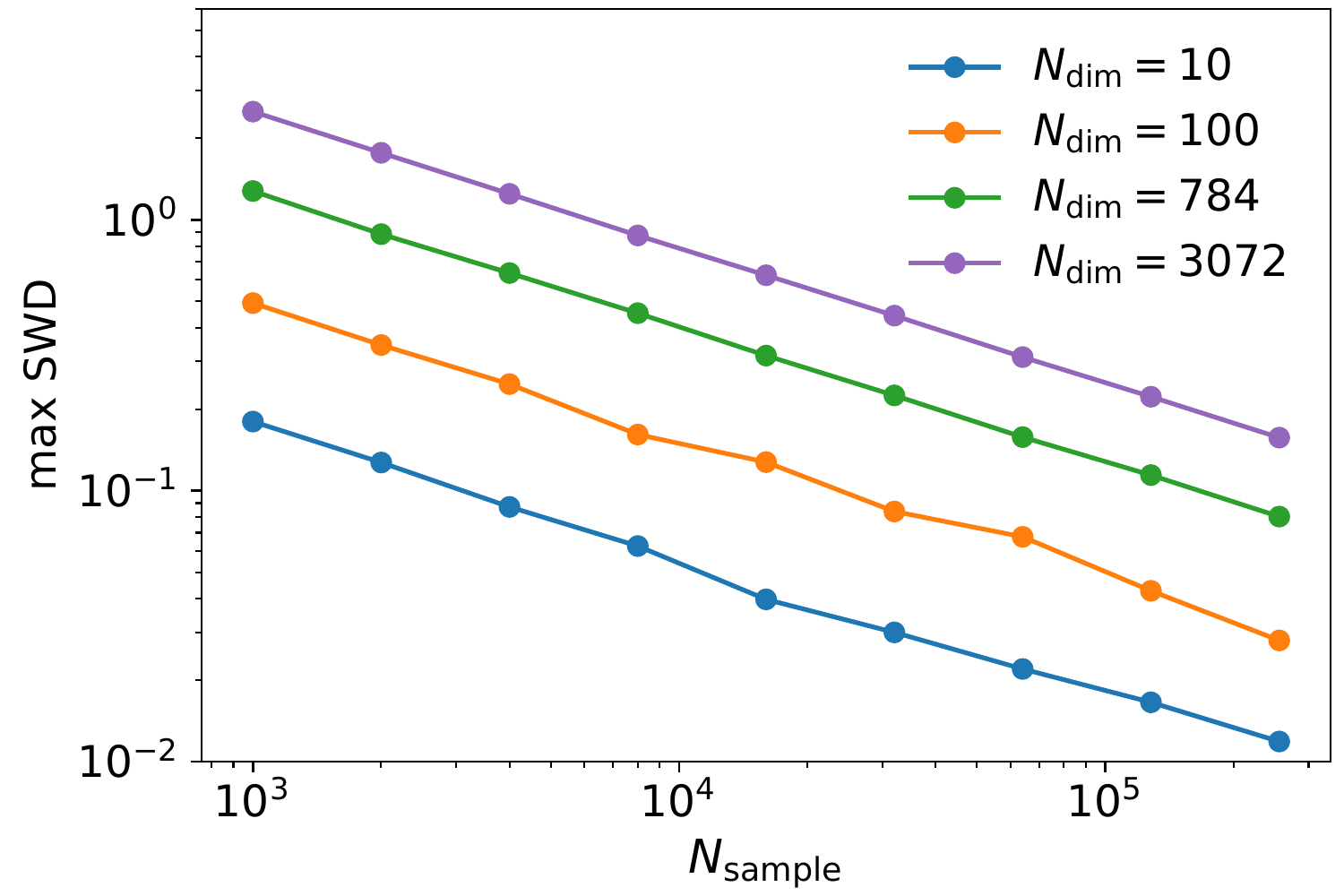}
     \caption{The measured maximum sliced Wasserstein distance between two Gaussian datasets as a function of number of samples. 10 different starting points are used to find the global maximum.}
    \label{fig:maxSWD_nsample}
    \vskip -0.10in
\end{figure}

We firstly test the convergence of SIG on MNIST dataset with different $K$ choices. We measure the SWD (Equation \ref{eq:SWp}) and max SWD (Equation \ref{eq:maxSWp}) between the test data and model samples for different iterations (without patch based hierarchical modeling). The results are presented in Figure \ref{fig:convergence}. The SWD is measured with 10000 Monte Carlo samples and averaged over 10 times. The max SWD is measured with Algorithm \ref{alg:KmaxSWD} ($K=1$) using different starting points in order to find the global maximum. We also measure the SWD and max SWD between the training data and test data, which gives an estimate of the noise level arising from the finite number of test data. For the range of $K$ we consider ($1\leq K\leq 128$), all tests we perform converges to the noise level, and the convergence is insensitive to the choice of $K$, but mostly depends on the total number of 1D transformations ($N_{\mathrm{iter}}\cdot K$). As a comparison, we also try running SIG with random orthogonal axes per iteration, and for MNIST, our greedy algorithm converges with two orders of magnitude fewer marginal transformations than random orthogonal axes (Figure \ref{fig:convergence}).

For $K=1$, the objective function (Equation \ref{eq:minimax}) is the same as max SWD, so one would expect that the max SWD between the data and the model distribution keep decreasing as the iteration number increases. For $K>1$, the max K-SWD is bounded by max SWD (Equation \ref{eq:max-K-SWD-bound1} and \ref{eq:max-K-SWD-bound2}) so one would also expect similar behavior. However, from Figure \ref{fig:convergence} we find that max SWD stays constant in the first 400 iterations. This is because SIG fails to find the global maximum of the objective function in those iterations, i.e., the algorithm converges at some local maximum that is almost perpendicular to the global maximum in the high dimensional space, and therefore the max SWD is almost unchanged. This suggests that our algorithm does not require global optimization of $A$ at each iteration: even if we find only a local maximum, it can be compensated with subsequent iterations. Therefore our model is insensitive to the initialization and random seeds. This is very different from the standard non-convex loss function optimization in deep learning with a fixed number of layers, where the random seeds often make a big difference \citep{lucic2018gans}.

In Figure \ref{fig:ablation} we show the samples of SIG of random axes, optimized axes and hierarchical approach. On the one hand, the sample quality of SIG with optimized axes is better than that of random axes, suggesting that our proposed objective max K-SWD improves both the efficiency and the accuracy of the modeling. On the other hand, SIG with optimized axes has reached the noise level on both SWD and max SWD at around 2000 marginal transformations (Figure \ref{fig:convergence}), but the samples are not good at that point, and further increasing the number of 1D transformations from 2000 to 200000 does not significantly improve the sample quality. At this stage the objective function of Equation \ref{eq:minimax} is dominated by the noise from finite sample size, and the optimized axes are nearly random, which significantly limits the efficiency of our algorithm. To better understand this noise, we do a simple experiment by sampling two sets of samples from the standard normal distribution $\mathcal{N}(0,I)$ and measuring the max SWD using the samples. The true distance should be zero, and any nonzero value is caused by the finite number of samples. In Figure \ref{fig:maxSWD_nsample} we show the measured max SWD as a function of sample size and dimensionality. For small number of samples and high dimensionality, the measured max SWD is quite large, suggesting that we can easily find an axis where the marginalized PDF of the two sets of samples are significantly different, while their underlying distribution are actually the same. Because of this sample noise, once the generated and the target distribution are close to each other (the max K-SWD reached the noise level), the optimized axes becomes random and the algorithm becomes inefficient. To reduce the noise level, one needs to either increase the size of training data or decrease the dimensionality of the problem. The former can be achieved with data augmentation. In this study we adopt the second approach, i.e., we effectively reduce the dimensionality of the modeling with a patch based hierarchical approach. The corresponding samples are shown in the bottom panel of Figure \ref{fig:ablation}. We see that the sample quality keeps improving after 2000 marginal transformations, because the patch based approach reduces the effective noise level.

\subsection{Effects of regularization parameter $\alpha$ in density estimation}

\begin{figure}[htb]
     \centering
     \begin{subfigure}[b]{0.495\linewidth}
         \centering
         \includegraphics[width=\linewidth]{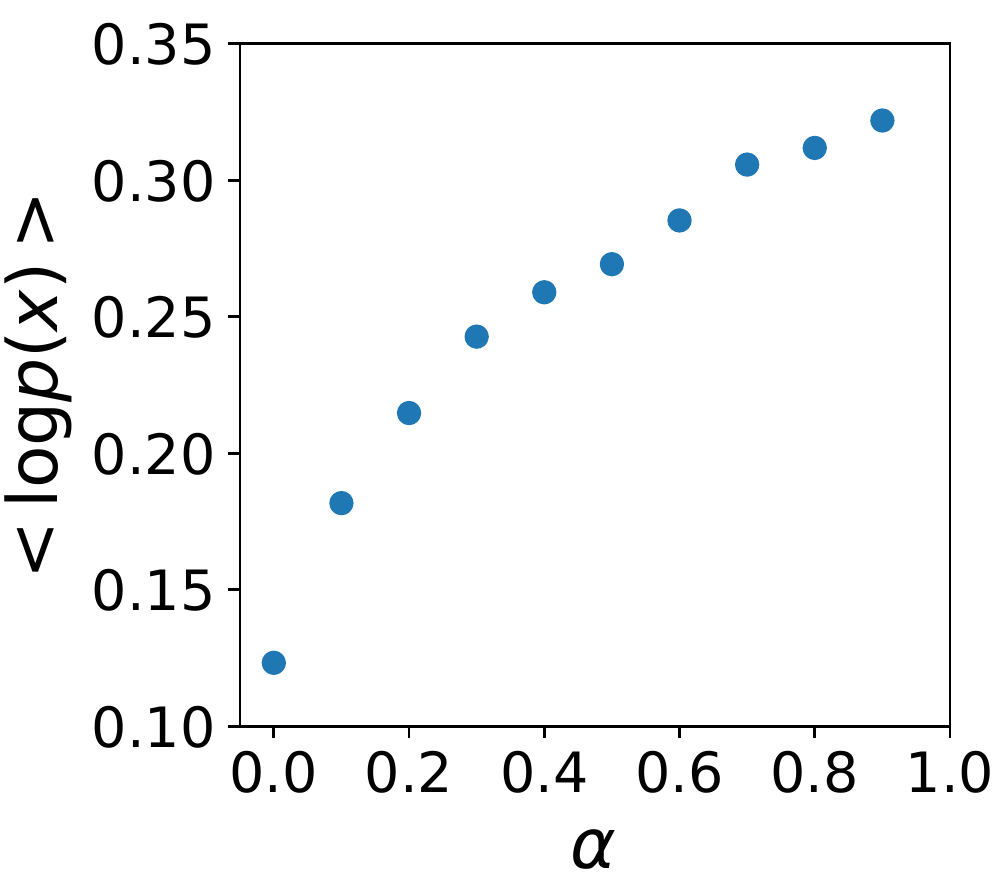}
     \end{subfigure}
     \hfill
     \begin{subfigure}[b]{0.495\linewidth}
         \centering
         \includegraphics[width=\linewidth]{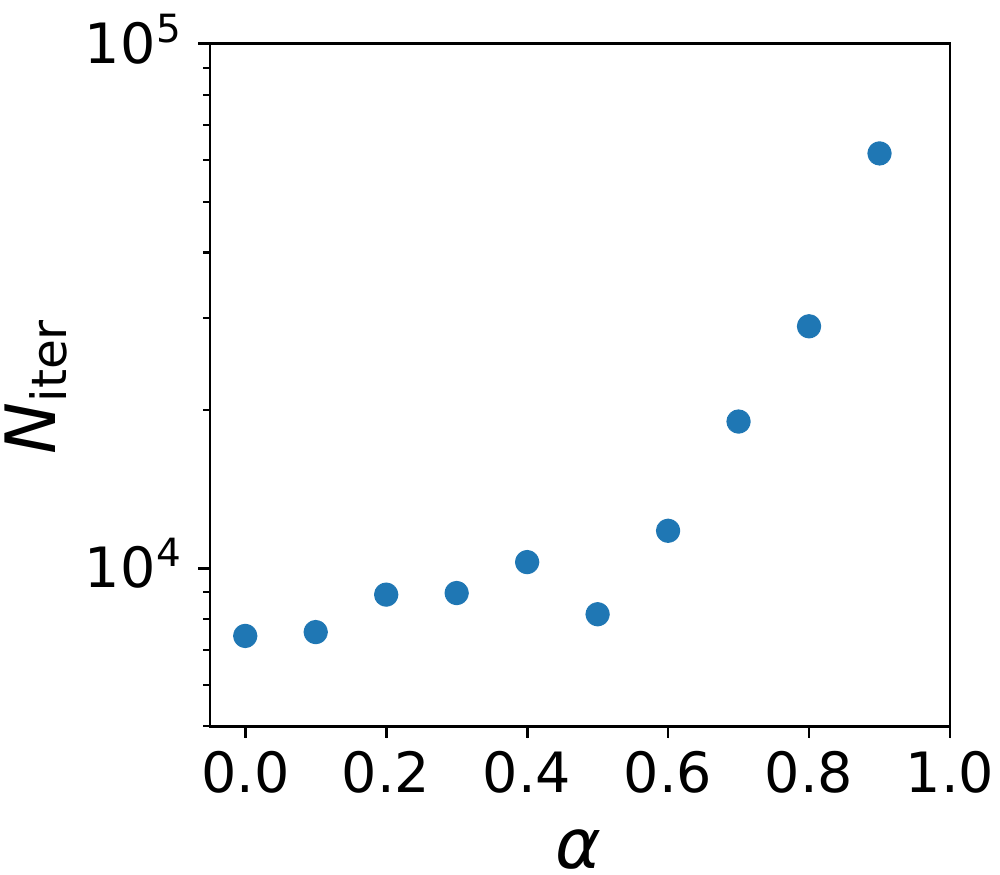}
     \end{subfigure}
     \caption{Test log-likelihood (left panel) and number of iterations (right panel) as a function of regularization parameter $\alpha$ on POWER dataset.}
     \label{fig:alpha}
     \vskip -0.15in
\end{figure}

To explore the effect of regularization parameter $\alpha$, we train GIS on POWER dataset with different $\alpha$. We keep adding iterations until the log-likelihood of validation set stops improving. The final test $\log p$ and the number of iterations are shown in Figure \ref{fig:alpha}. We see that with a larger $\alpha$, the algorithm gets better density estimation performance, at the cost of taking more iterations to converge. Setting the regularization parameter $\alpha$ is a trade-off between performance and computational cost.

\section{Experimental details}
\label{sec:detail}

\begin{table*}[htb]
  \caption{GIS hyperparameters for density-estimation results in Table \ref{tab:density}.}
  \label{tab:hyper_GIS}
  \vskip 0.15in
  \centering
  \begin{tabular}{>{\centering}c|>{\centering}c>{\centering}c>{\centering}c>{\centering}c>{\centering}c|>{\centering}m{0.09\linewidth}>{\centering\arraybackslash}m{0.09\linewidth}}
    \toprule
    Hyperparameter & POWER & GAS & HEPMASS & MINIBOONE & BSDS300 & MNIST & Fashion\\
    \midrule\midrule
    $K$ & 6 & 8 & 8 & 8 & 8 & 8\ ($q=4$)\ 4\ ($q=2$) & 8\ ($q=4$)\ 4\ ($q=2$)\\
    $\alpha=(\alpha_1, \alpha_2)$ & (0.9,0.9) & (0.9,0.9) & (0.95, 0.99) & (0.95, 0.999) & (0.95, 0.95) & (0.9, 0.99) & (0.9, 0.99)\\
    $b$ & 2 & 1 & 1 & 2 & 5 & 1 & 1\\
    \bottomrule
  \end{tabular}
  \vskip -0.1in
\end{table*}

\begin{table*}[htb]
  \caption{The architectures of SIG for modeling different image datasets in Section \ref{subsec:samples}. The architecture is reported in the format of $(q^2\cdot c,K)\times L$, where $q$ is the side length of the patch, $c$ is the depth of the patch, $K$ is the number of marginal transformations per patch, and $L$ is the number of iterations for that patch size. MNIST and Fashion-MNIST share the same architecture.}
  \label{tab:hyper_SIG}
  \vskip 0.15in
  \centering
  \begin{tabular}{>{\centering}c|>{\centering}c>{\centering}c>{\centering\arraybackslash}c}
    \toprule
    & MNIST / Fashion-MNIST & CIFAR-10 & CelebA\\ 
    \midrule\midrule
     \multirow{16}{*}{architecture} & $(28^2\cdot 1,56)\times 100$ & $(32^2\cdot 3,64)\times 200$ & $(64^2\cdot 3,128)\times 200$\\
    & $(14^2\cdot 1,28)\times 100$ & $(16^2\cdot 3,32)\times 200$ & $(32^2\cdot 3,64)\times 200$\\
    & $(7^2\cdot 1,14)\times 100$ & $(8^2\cdot 3,16)\times 200$ & $(16^2\cdot 3,32)\times 200$\\
    & $(6^2\cdot 1,12)\times 100$ & $(8^2\cdot 1,8)\times 100$ & $(8^2\cdot 3,16)\times 200$\\
    & $(5^2\cdot 1,10)\times 100$ & $(7^2\cdot 3,14)\times 200$ & $(8^2\cdot 1,8)\times 100$\\
    & $(4^2\cdot 1,8)\times 100$ & $(7^2\cdot 1,7)\times 100$ & $(7^2\cdot 3,14)\times 200$\\
    & $(3^2\cdot 1,6)\times 100$ & $(6^2\cdot 3,12)\times 200$ & $(7^2\cdot 1,7)\times 100$\\
    & $(2^2\cdot 1,4)\times 100$ & $(6^2\cdot 1,6)\times 100$ & $(6^2\cdot 3,12)\times 200$\\
    & & $(5^2\cdot 3,10)\times 200$ & $(6^2\cdot 1,6)\times 100$\\
    & & $(5^2\cdot 1,5)\times 100$ & $(5^2\cdot 3,10)\times 200$\\
    & & $(4^2\cdot 3,8)\times 200$ & $(5^2\cdot 1,5)\times 100$\\
    & & $(4^2\cdot 1,4)\times 100$ & $(4^2\cdot 3,8)\times 200$\\
    & & $(3^2\cdot 3,6)\times 200$ & $(4^2\cdot 1,4)\times 100$\\
    & & $(3^2\cdot 1,3)\times 100$ & $(3^2\cdot 3,6)\times 200$\\
    & & $(2^2\cdot 3,4)\times 200$ & $(3^2\cdot 1,3)\times 100$\\
    & & $(2^2\cdot 1,2)\times 100$ & $(2^2\cdot 3,6)\times 100$\\
    \midrule
    Total number of iterations $L_{\mathrm{iter}}$ & 800 & 2500 & 2500\\
    \bottomrule
  \end{tabular}
  \vskip -0.1in
\end{table*}

The hyperparameters of GIS include the number of axes per iteration $K$, the regularization $\alpha$, and the KDE kernel width factor $b$. We have two different $\alpha$ values: $\alpha=(\alpha_1, \alpha_2)$, where $\alpha_1$ regularizes the rational quadratic splines, and $\alpha_2$ regularizes the linear extrapolations. The KDE kernel width $\sigma$ is determined by the Scott's rule \citep{scott2015multivariate}:
\begin{equation}
    \sigma = b N^{-0.2}\sigma_{\mathrm{data}} ,
\end{equation}
where $N$ is the number of training data, and $\sigma_{\mathrm{data}}$ is the standard deviation of the data marginalized distribution.

The hyperparameters for density-estimation results in Table \ref{tab:density} are shown in Table \ref{tab:hyper_GIS}. $K$ is determined by $K=\min(8,d)$. For BSDS300 we first whiten the data before applying GIS. For high dimensional image datasets MNIST and Fashion-MNIST, we add patch-based iterations with patch size $q=4$ and $q=2$ alternately. Logit transformation is used as data preprocessing. For all of the datasets, we keep adding iterations until the validation $\log p$ stops improving.

For density estimation of small datasets, we use the following hyperparameter choices for large regularization setting: $b=1$, $K=\min(8,d)$,\\ $\alpha=(1-0.02\log_{10}(N_{\mathrm{train}}), 1-0.001\log_{10}(N_{\mathrm{train}}))$. While for low regularization setting we use $b=2$ and $\alpha=(0, 1-0.01\log_{10}(N_{\mathrm{train}}))$. The size of the validation set is $30\%$ of the training set size. All results are averaged over 5 different realizations.

\begin{figure}[t]
     \centering
      \includegraphics[width=\linewidth]{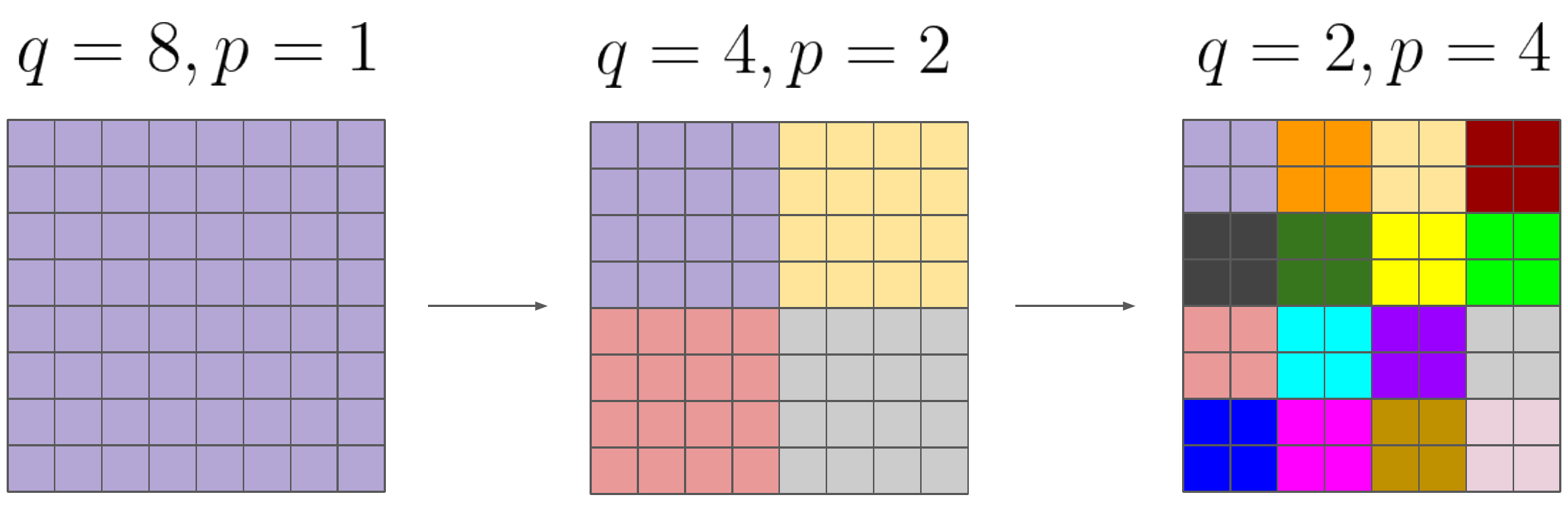}
     \caption{Illustration of the hierarchical modeling of an $S=8$ image. The patch size starts from $q=8$ and gradually decreases to $q=2$.}
     \label{fig:patch_hierarchy}
     \vskip -0.1in
\end{figure}

The hyperparameters of SIG include the number of axes per iteration $K$, and the patch size for each iteration, if the patch-based approach is adopted. We show the SIG hyperparameters for modeling image datasets in Table \ref{tab:hyper_SIG}. As discussed in Section \ref{subsec:patch}, the basic idea of setting the architecture is to start from the entire image, and then gradually decrease the patch size until $q=2$. An illustration of the patch-based hierarchical approach is shown in Figure \ref{fig:patch_hierarchy}. We set $K=q$ or $K=2q$, depending on the datasets and the depth of the patch. For each patch size we add $100$ or $200$ iterations.

For OOD results in Section \ref{subsec:ood}, we train SIG and GIS on Fashion-MNIST with $K=56$. GIS is trained with $b=1$ and $\alpha=0.9$ (the results are insensitive to all these hyperparameter choices). We do not use logit transformation preprocessing, as it overamplifies the importance of pixels with low variance. The number of iterations are determined by optimizing the validation $\log p$. For SIG, which cannot produce good $\log p$, 
the results shown in Table \ref{tab:auroc} use 100 iterations, but we verify they do not depend on this choice and are stable up to thousands of iterations. 

%

\end{document}